\documentclass[format=acmsmall,
nonacm
]{acmart}
\AtBeginDocument{%
  }

\usepackage{booktabs}
\usepackage{subcaption}

\usepackage{lineno}

\usepackage{algorithm}
\usepackage{algpseudocode} 
\usepackage{wrapfig} 
\newcommand{\hlhref}[2]{\href{#1}{\textcolor{black}{\textcolor{teal}{#2}}}}

\newcommand{\B}{\ensuremath{\mathsf{B}}}

\newcommand{\s}{\ensuremath{\mathbf{s}}}

\newcommand{\x}{\ensuremath{\mathbf{x}}}
\newcommand{\X}{\ensuremath{\mathbf{X}}}

\renewcommand{\S}{\ensuremath{\mathbf{S}}} 

\newcommand{\R}{\ensuremath{\mathbf{R}}}

\begin{document}

\title{Swimming with Whales: \\ 
Analysis of
Power Imbalances in 
Stake-Weighted 
Governance}

\author{Yuzhe Zhang}
\email{zhangyuzhe1212@gmail.com}
\affiliation{%
  \institution{Independent researcher}
  \city{Melbourne}
  \country{Australia}}
\author{Manvir Schneider}
\email{manvir.schneider@cardanofoundation.org}
\affiliation{%
  \institution{Cardano Foundation}
  \city{Zurich}
  \country{Switzerland}
}

\author{Qin Wang}
\affiliation{%
  \institution{CSIRO}
  \city{Sydney}
  \country{Australia}}
\email{qinwangtech@gmail.com}

\author{Davide Grossi}
\affiliation{%
  \institution{University of Groningen and University of Amsterdam}
  \city{Groningen}
  \country{The Netherlands}
}
\email{d.grossi@rug.nl}

\renewcommand{\shortauthors}{Zhang et al.}

\begin{abstract}
Voting methods weighted by stakes are the fundamental governance paradigm in Proof-of-Stake (PoS) blockchains.
Such a paradigm is known to be prone to power distortions:
a few users possessing large stakes may completely control decision making, even without owning the totality of the stakes.
We study this phenomenon through the lens of computational social choice, focusing on the extent of power imbalances in stake-weighted voting when power is quantified using the Penrose-Banzhaf power index.
Our work presents both analytical and empirical contributions.
Analytically, we demonstrate that while a perfect alignment between power and relative stake ownership is generally unattainable, it can be approximated in expectation under specific conditions. Empirically, using data from a real-world on-chain governance system (Project Catalyst), we provide a more fine-grained understanding of the power imbalances that are likely to occur in current stake-weighted governance systems.
\end{abstract}

\maketitle

\section{Introduction}
Our work focuses on algorithmic governance \cite{kiayias2022sok} in PoS blockchain platforms. We are concerned in particular with on-chain treasury governance, that is, decisions about the allocation of funds (in native cryptocurrency) to projects proposed by stakeholders in a blockchain community. A notable example of such on-chain governance systems is \href{https://projectcatalyst.io/}{Project Catalyst}, the on-chain treasury project of the Cardano blockchain, which we are going to use as the motivating example for our work.  Specifically, we are interested in applying insights and methods from computational social choice \cite{brandt2016handbook} to better understand the voting methods deployed in systems such as Project Catalyst for the selection of community-proposed projects and, ultimately, to contribute to the improvement of those methods. As such, our work contributes to a line of research at the interface of blockchain and computational social choice \cite{grossi2022social}.

\paragraph{Context: participatory budgeting in treasury systems.} Project selection in treasury systems like Project Catalyst is an instance of the class of social choice problems known as {\it participatory budgeting} problems \cite{rey2023computational}: members of the community submit projects with associated costs for consideration by fellow members; members vote for the projects they approve of, their votes carrying weights equal to the members' stakes; finally, based on the expressed votes, winning projects are selected. 

There are, however, two key features of project selection in Catalyst that take them apart from the participatory budgeting theory developed thus far by the computational social choice community. 
{\it First}, unlike in standard democratic contexts but as common in blockchain governance, voting in Catalyst is weighted by the stakes of the voter (in ADA, the native cryptocurrency of the Cardano blockchain). {\it Second}, an underlying assumption of the current models of participatory budgeting is that voters are able to consider all projects for approval\footnote{There is some recent work which looked at relaxing such assumptions, however. See, e.g., \cite{goyal2023low}.}, however, this is difficult to hold in Catalyst due to its vast number of proposals. 

In a standard Catalyst funding round, however, the number of projects proposed is way larger than any single voter can possibly individually consider. The last funding round (Fund 13) has seen more than 1600 submitted proposals, with 12\% ultimately selected for funding. Fund 13 had six categories. Across these categories, the percentage of ``Yes'' votes relative to the total registered stake and the total participating stake provides insight into the top projects' influence and voter engagement. In general, leading projects in each category received only small percentages of registered stake, often ranging from 6\% to 11\% and 10\% to 21\% relative to registered stake and participating stake, respectively\footnote{\url{https://docs.google.com/spreadsheets/d/1t\_fKMmOu28ayIGNsv3TYujM7A47YeGglj5XUXaEHXSs/edit?usp=sharing}}. So, in a typical Catalyst funding round, only a small fraction of users---and therefore of total stakes---approve of any single project.

\paragraph{Problem: disproportionate influence of large voters.} 
Currently, the voting method used in Catalyst to select winning projects chooses the set of projects that maximize the sum of the stakes of the voters approving the projects, while remaining within budget. This participatory budgeting voting method, also known as greedy approval voting, has known drawbacks from a social choice perspective \cite{rey2023computational}. Our focus, however, is that the use of such method in combination with the two features of the Catalyst participatory budgeting problem that we highlighted above, creates wide voting power discrepancies, which may be considered undesirable from a blockchain governance perspective. Because any selected project elicits only a small fraction of total approvals, and because the stake distribution in the system is highly unequal,\footnote{This is a widely discussed centralization-related issue for proof-of-stake systems \cite{fanti2019compounding,kiayias2022sok}, and certainly relevant also to Cardano.} whether one given project makes it to the final bundle of winning projects is often down even to just one single voter. 

Among many other examples, our same research project 
is a case in point. It was funded through Project Catalyst Fund 13, 
with 92\% of votes coming from a single whale voter. 
Still, three other projects that also received votes from the same whale voter were not funded due to insufficient remaining budget. 
Importantly, in Fund 13 there were projects that received strong support from the broader community but were not backed by major whales. The most notable example is a project that received over 500 million ``yes'' votes from the community and got funded. Yet it was not supported by the largest whale, who held a stake of approximately 180 million. Given the current voting method deployed by Catalyst, some users appear to wield voting power that is considerably out of proportion with the---even very large---share of stakes they control.

\medskip

Which projects to fund through Catalyst is a governance problem. We consider sound blockchain governance, using the definition put forth in \cite{kiayias2022sok}, to be:
\begin{quotation}
    "\ldots the ability of a blockchain platform community members to express their will effectively regarding the future evolution of the platform as well as the best possible utilization of its resources." 
\end{quotation}
In this view, disproportionate voting power in treasury decisions can arguably be considered problematic. In this paper we construct evidence towards the above observations and advance proposals for more stake-proportional forms of governance for Project Catalyst.

\paragraph{Contribution and plan of the paper.}
After having established the necessary preliminaries in Section~\ref{sec:bkg}, the paper proceeds in two steps. We first formally define the problem and provide a general limitative result about power balance in stake-weighted governance in Section~\ref{sec:determine}. We then provide an analytical framework to understand the extent of power imbalances under the currently predominant governance paradigm in which stakes are treated directly as voting weighths in Section~\ref{sec:probabilistic}.
Finally, Section \ref{sec:experiments} applies our approach to simulated data, providing further insights into the behavior of power imbalances in current blockchain governance systems.

\section{Preliminaries}
\label{sec:bkg}

We will use three major building blocks, which are used to analyze voting power in Catalyst. We formalize them as follows.

\subsection{Weighted Approval Voting}
\label{subsec:WAV}

\paragraph{Setting.} Let $N=\{1,2,\cdots, n\}$ be a set of {\it users}, also referred to as {\it voters} or {\it agents}. In what follows, for simplicity, we assume $n$ to be odd.
These users need to decide whether to accept (alternative $1$) or reject (alternative $0$) a given project via voting.
Each user $i\in N$ owns an amount of {\it stakes}, denoted as $s_i\in \mathbb{R}_{\ge 0}$, and $\mathbf{s}=(s_1,\cdots, s_n)$ is the vector of all stakes, called a {\it stake profile}. 
Let then $\mathbf{v}=(v_1,\cdots, v_n) \in \{0,1\}^n$ be the vector of the users' votes/ballots, called {\it vote profile}, which records the ballot cast by each user.
We furthermore denote by $\mathcal{S}$ and $\mathcal{V}$ the sets of all stake and vote profiles, respectively.
Then, a function $f: \mathcal{S}\times \mathcal{V}\rightarrow \{0,1\}$ is a {\it voting rule}, taking as input a stake profile and a vote profile and determining a {\it social choice}.

In all the funding rounds held so far, Project Catalyst has used approval voting to determine the selection of winning projects. For each project $p \in P$, where $P$ is the set of projects submitted to the funding round ($|P| = m$), an approval-based election is held. Let us denote the voting profile for project $p$ by $\mathbf{v}^p\in \{0,1\}^N$. Given a stake profile $\mathbf{s}$ and a voting profile $\mathbf{v}^p$ for each project $p$, an approval score $\sigma: P \to \mathbb{R}_{\geq 0}$ is computed for each project $p_k$ such that $\sigma(p) = \sum_{1 \leq i \leq n} v^p_i \cdot s_i$. That is, the approval score of a project is the sum of the stakes of the users approving the project. Projects are ordered by their approval scores and selected in decreasing order of score until the budget is exhausted. This is a weighted version of the participatory budgeting, see \cite{rey2023computational}. So, assuming a set $W$ of $k-1$ projects have been selected, at round $k$ a budget-feasible project $p \in P \backslash W$ is added to the winning selection only if
\begin{equation} \label{eq:quota}
    \sigma(p) \geq \underbrace{\max_{p' \in P \backslash W} \left( \sigma(p') \right)}_{\theta^k}.
\end{equation}
That is, at each approval round $k$ a project $p$ is selected only if its approval score meets a quota $\theta^k$.\footnote{Uniformly random tie breaking is assumed throughout the paper.}

The limitations and drawbacks of the unweighted version of the above method have already been extensively discussed in the computational social choice literature (see \cite{rey2023computational} for a recent overview).
In what follows we scrutinize its weighted version further, as the one being most predominantly in use in PoS blockchain governance. We zoom in on the quota-driven decision element of the method that we highlighted in Equation \eqref{eq:quota}, with the aim of analyzing the allocation of voting power the method induces, and the extent to which they introduce imbalanced influence in collective decisions.

\subsection{Weighted Quota Rules} 
\label{subsec:quota}

Assume we are deciding whether a given project $p$ is to be added to the set of winning projects at the current round of the local optimization-based approval method described above. This amounts essentially to a decision about whether a project meets a certain quota of support or not. So, let us start by defining the class of weighted quota rules (WQR) in its full generality. Each voting rule $f$ in this class can be seen as a composition of two functions: 
\begin{itemize}
    \item a voting weight allocation function (VWA) $w: \mathcal{S}\rightarrow \mathbb{R}^n_{\ge 0}$;
    \item and a quota function $q_\theta: \mathbb{R}^n_{\ge 0}\times \mathcal{V}\rightarrow (\mathbb{R}_{\ge 0}\rightarrow \{0,1\})$.
\end{itemize}
That is, given a {\it quota} $\theta\in (\frac{1}{2}, 1]$, $q_\theta (w(\mathbf{s}),\mathbf{v})=1$ if
\begin{equation}
\frac{\sum_{i\in N}w_i(\mathbf{s})v_i}{\sum_{i\in N}w_i(\mathbf{s})}\ge \theta,\footnote{We are therefore working with so-called `relative' quota.}
\end{equation}
otherwise, $q_\theta (w(\mathbf{s}),\mathbf{v})=0$.
Intuitively, $f$ first assigns a {\it voting weight} to each user based on their stakes, and then outputs a {\it collective decision} based on each user's voting weight and ballot. We call a tuple $\langle N,\mathbf{s}, f\rangle$ a {\it voting instance}.




\subsection{Voting Power}

Given a voting instance, we can uniquely decide whether the final voting result is consistent with the choice of a coalition (i.e., a subset of $N$) given that its complementary set in $N$ chooses the other alternative.
Following this intuition, a voting instance can map to a unique {\it simple game} defined as follows.

\begin{definition}[Simple game]
A simple game is a tuple $g=\langle N, \mathbf{s}, f_g \rangle$, where $N$ is a set of users, $\mathbf{s}$ is a stake profile, and $f_g: 2^N\rightarrow \{0,1\}$ is a monotone {\it decision function} corresponding to the weighted quota voting rule $f$.
That is, given a coalition $C\subseteq N$, $f_g(C)=1$ if $\frac{\sum_{i\in C}w_i(\mathbf{s})}{\sum_{i\in N}w_i(\mathbf{s})}\ge \theta$, where $\theta$ and $w$ are the quota and weight allocation function of $f$.
$C$ is {\it winning} if $f_g(C)=1$, otherwise it is losing.
\end{definition}
We recall that a decision function $f$ is {\it monotone} if for any pair of coalitions $\{C,C' \} \subseteq N$ such that $C'\subseteq C$, $f(C')\le f(C)$.

Based on a simple game, we are able to exactly measure each user's {\it voting power}, e.g., by using the Banzhaf power index.

\begin{definition}[Banzhaf index] \label{def:banzhaf}
Given a simple game $g=\langle N,\mathbf{s}, f_g\rangle$ where $f_g$ corresponds to a weighted quota voting rule $f$, the Banzhaf power index based on $f$ of user $i\in N$ is a mapping $\mathsf{B}^f_i: \mathbb{Z}^N\rightarrow (N\rightarrow \mathbb{R}_{\ge 0})$ that takes a stake profile and outputs a non-negative real number less than  $1$ for each agent.
Formally, for each user $i\in N$, their Banzhaf power index is
\begin{align}
\mathsf{B}^f_i (\mathbf{s})=\frac{1}{2^{n-1}}\sum_{C\subseteq N\setminus\{i\}}(f(C\cup\{i\})-f(C)).
\end{align}
\end{definition}
We call agent $i$ a swing agent or a pivotal agent for coalition $C$ if $f(C\cup\{i\})-f(C) = 1$.

\section{On Balanced Voting Power}
\label{sec:balance}

In this section we introduce the main metric of power (im)balance we are going to use in the paper: the ratio between power---as measured by the Banzhaf index---and stakes. We first establish a general limitative result about such ratio and then derive expressions that capture its variance under general assumptions on how stakes are distributed in a population of users. This latter result gives us a precise understanding of the level of power imbalances that a user should expect in a system, when voting weights are proportional to stakes, as in standard PoS governance systems such as Project Catalyst.

\subsection{Stake-Proportional Voting Power}\label{sec:determine}

Intuitively, an ideal WQR would be one based on which each user achieves a Banzhaf voting power that is proportional to their stakes, relative to the totality of stakes in the system.
That is,
a WQR $f$ 
such that in each simple game $g=\langle N, \mathbf{s}, f_g\rangle$, for each user $i\in N$, $i$'s Banzhaf power index satisfies
\begin{align}\label{eq:proportional}
\mathsf{B}^f_i(\mathbf{s})\propto \frac{s_i}{\sum_{j\in N}s_j}.
\end{align}
or equivalently, that
$     
\frac{\mathsf{B}^f_i(\mathbf{s})}{s_i} = 1.
$
We refer to $\frac{\mathsf{B}^f_i(\mathbf{s})}{s_i}$ as the {\em power-stake ratio} of $i$ under $f$. We call a WQR $f$ that satisfies Equation \eqref{eq:proportional} {\it perfectly balanced}. The following simple result shows that, as one might expect, perfectly balanced WQRs do not exist in general.

\begin{theorem}\label{thm:no_wqr}
There exists no perfectly balanced WQR. 
\end{theorem}
\begin{proof}
We proceed towards a contradiction. Assuming a perfectly balanced WQR exists, we construct a stake profile that fails to generate a proportional Banzhaf voting power profile.
Assume that we have a WQR $f$, that consists of a quota $\theta$ and a VWA $w$, based on which a stake profile $\s^*$ generates a proportional Banzhaf voting power profile.
We additionally assume that $s^*_j>0$ and $w_j(\s^*)<\theta$ for all $j\in N$.
That is, $f$ and $\s^*$ satisfy the proportional criterion in Equation \ref{eq:proportional}.
In what follows, we show that we can construct a different stake profile $\s^*$ which cannot generate a proportional Banzhaf power profile under $f$.

For an arbitrary user $i\in N$, we increase her stake number to a large number and maintain the stake number of the other users' stake number constant, obtaining stake profile $\s$.
That is, for each user $j\in N\setminus\{i\}$, $s_j = s^*_j$, and $s_i > s^*_i$.
Since the VWA is weakly monotonic, we have that $w_i(\s)\ge w_i(\s^*)$.
Then, we consider two exhaustive cases: (1) $w_i(\s)\ge \theta$, and (2) $w_i(\s) < \theta$.

\noindent (1) We have that user $i$ becomes a dictator, i.e., $\B^f_i(\s)=1$ and $\B^f_j(\s)=0$ for each $j\in N\setminus \{i\}$.
Apparently, we have that the Banzhaf voting power profile based on $f$ for $\s$ is not proportional, since each user has a positive stake number.

\noindent (2) In this case, it is possible that $i$ is still a dictator, and then, the case becomes equivalent to (1).
Therefore, we only consider that $i$ is not a dictator.
Since $\s$ can be any stake profile with condition that other users than $i$ have the same stake number, we can assume that $s^{max}_{-i}/s_i < 1/2^{n-1}$, where $s^{max}_{-i}=\max_{j\in N\setminus\{i\}}s_j$ is the maximal stake number of other users than $i$.
Since $i$ is not a dictator, there must exist an user $j\in N$, who has a positive Banzhaf power index.
By Definition \ref{def:banzhaf}, we have that $\B^f_j(\s)\ge 1/2^{n-1}$ for each $i\in N\setminus\{i\}$, and $\B^f_i(\s)< 1$.
We thus obtain a contradiction, since:
$$
\frac{\B^f_j(\s)}{\B^f_i(\s)}>\frac{1}{2^{n-1}}>\frac{s_j}{s_i}.
$$
This completes the proof.
\end{proof}


The following example illustrates the reach of the theorem, showing how it bears on various weight allocation functions including also quadratic ones as proposed in \citet{penrose1946elementary,lalley2018quadratic}.
\begin{example}\label{ex:1}
Consider an instance with 5 agents $N = \{1,2,3,4,5\}$, and the stake profile is $\s = (10, 90, 100, 200, 600)$. 

First, we consider a linear VWA is used in WQR $f_1$, i.e., each agent's voting weight is proportional to their stakes, and the quota is $\theta = 0.5$, i.e., the weight majority rule.
Then, in this case, agent $5$'s weight is slightly larger than the quota, and hence, she is a dictator with a Banzhaf index of $\B_5^{0.5}(\s) = 1$. The other agents have a Banzhaf index of $0$, since they cannot swing for any coalition.
The corresponding Banzhaf index profile $\mathbf{\B}(\s) = (0,0,0,0,1)$ fails to achieve the perfect balance.

In this example, agent $5$ attains too much voting power through the WQR $f_1$.
The well known Penrose voting scheme \cite{penrose1946elementary}, also known as the quadratic voting scheme \cite{Lalley2018} has been proposed as a mechanism to restrain the voting power of such ``super-voters''. Let $f_2$ be the WQR that uses the Penrose voting scheme, and the corresponding VWA produces a voting weight profile $\mathbf{w} = (3.16, 9.49, 10, 14.14, 24.49)$, with a total weight of $61.28$.
Following the weighted majority rule, we obtain that the Banzhaf index profile is $\B^{f_2}(\mathbf{w}) = (0, 0.25, 0.25, 0.25, 0.75)$.
Agent $5$'s voting power is reduced when using $f_2$, however, perfect balance is not attained.

It is also worth mentioning that, although the Penrose scheme is effective in limiting the power of ``super-voters'', it is not easily applicable to blockchain governance because of wallet splitting. For example, if agent $5$ divides her stakes into two equivalent accounts, each with $300$ stakes, then, the stake profile becomes $\s' = (10,90,100,200,300,300)$ for a new agent set $N' = \{1,2,3,4,5,5'\}$.
Under $f_2$, the weight profile becomes $\mathbf{w}' = (3.16, 9.49, 10, 14.14, 17.32, 17.32)$, and we compute the new Banzhaf index profile as $\B^{f_2}(\mathbf{w}') = (0.0625, 0.3125, 0.3125, 0.3125, 0.4375, 0.4375)$.
This strategy yields a total voting power of $0.4375 \times 2 = 0.875$ for agent $5$, providing space for manipulation.
\end{example}

It is worth observing that Theorem \ref{thm:no_wqr} can be viewed as a contribution to the inverse Banzhaf index problem \cite{alon2010inverse}: given a desired Banzhaf index distribution, find a weighted voting game that yields a Banzhaf distribution as close as possible to the target one. The literature on the inverse Banzhaf index problem had already observed that perfectly balanced WQR are unachievable in general due to the discrete nature of the Banzhaf index, especially in small instances \cite{kurz2012inverse}. However, to the best of our knowledge the impossibility stated in Theorem \ref{thm:no_wqr} has never been formalized exactly, and in full generality, in the literature. Also, our proofs of the theorem shows that the impossibility is not linked to small electorates but can occur in electorates of arbitrary size.

The impossibility result of Theorem \ref{thm:no_wqr} implies that no WQR can possibly guarantee perfect balance of power: that is, a distribution of power that is proportional to stakes. As the ideal power-stake ratio cannot be obtained we concentrate on understanding the extent to which the current mainstream choice of weight allocation function (that is, the identity function $id$) deviates from the a balanced power-stake ratio. This will be done in two steps. {\em First}, under the assumption that weights are linear functions of stakes (that is, linear VWA functions), we study the expectation that an arbitrary user would have with respect to their power-stake ratio (the ratio between their Banzhaf voting power and their stakes) in a stake-based voting system, under further assumptions on the distribution of stakes in the system. {\em Second}, we complement this theoretical analysis by providing an extensive computational study examining the distribution of power-stake ratios across voting instances that are structurally similar to real-world elections in Project Catalyst.

\subsection{Expected Power-Stake Ratio}\label{sec:probabilistic}

In the following contents, we only consider linear VWA functions, i.e., each agent's voting weight is proportional to their stake number.
Due to page limit, we move the full proofs of Lemma \ref{lem:uniform_distribution} and Corollary \ref{cor:uniform} to Appendix \ref{sec:verify}.

\subsubsection{Distributional assumptions}\label{sec:distribution_assumption}
Our analysis of the expected power-stake ratio assumes that individual stakes in the stake profile are i.i.d. with respect to a Gamma distribution $\Gamma (\alpha, \beta)$. We will then treat the individual stakes of each $i$ as a random variables $S_i$ (not necessarily normalized) and the stake profile ${\bf S}$ as a vector-valued random variable. Given that we are providing an analysis of linear weight-allocation functions, it follows that the corresponding normalized weight profile $\X$ 
follows a Dirichlet distribution $Dir(\boldsymbol{\alpha})$, where vector $\boldsymbol{\alpha} = (\alpha, \cdots, \alpha)$ collects the parameters of the distribution \cite{sethuraman1994constructive}. 
Let $X_{N'} = \sum_{i\in N'}X_i$ for each $N'\subseteq N$.
The identical parameters are due to the fact that we draw $S_i$ i.i.d. from $\Gamma (\alpha, \beta)$. Observe that when $\Gamma (\alpha, \beta)$ takes special parameters $\Gamma (1,1)$, it degenerates to the exponential distribution, and we obtain $Dir(1, \cdots, 1)$ which is the multi-variate uniform distribution.

More precisely, let $\x$ be a realisation of $\X$, the density of the Dirichlet distribution is:
$$P_{\X}(\x;\boldsymbol{\alpha}) = \frac{1}{B(\boldsymbol{\alpha})}\prod_{i\in N}x^{\alpha - 1}_i,$$
where $B(\boldsymbol{\alpha})$ is the multi-variate Beta function, i.e.,
$$
B(\boldsymbol{\alpha}) = \frac{\prod_{i\in N}\Gamma (\alpha_1)}{\Gamma (\sum_{i\in N}\alpha_i)} = \frac{\Gamma (\alpha)^n}{\Gamma(n\alpha)}.
$$



Before proceeding, it is worth elaborating on why we select the Gamma distribution as the probabilistic model underpinning our analysis. We do that for two reasons: (1) By tuning its shape ($\alpha$) and scale ($\beta$) parameters, the Gamma distribution is able to model a wide range of practically relevant stake distributions, and we are going indeed to fit the distribution on real-world data. (2) Gamma distributed components naturally form a Dirichlet distribution under normalisation, independently of the scale parameter $\beta$. Importantly, and as noted above, the Dirichlet distribution generalizes the multi-variate uniform distribution of weights considered in related work \cite{jelnov2014voting}, allowing us to retrieve earlier results as special cases while developing our analysis.
Lemma \ref{lem:uniform_distribution} below clarifies the relationship between Dirichlet and uniform multi-variate distributions, which have been used in previous work \cite{jelnov2014voting}.\footnote{The lemma will be used to lay a precise connection with existing results in Corollary \ref{cor:uniform}.}
\begin{lemma}\label{lem:uniform_distribution}
For the $n$-variate Dirichlet distribution $Dir(\boldsymbol{\alpha})$, setting $\boldsymbol{\alpha} = (1, \cdots, 1)$ degenerates $Dir(\boldsymbol{\alpha})$ to the $n$-variate uniform distribution.
\end{lemma}
\begin{proof}
We show the Dirichlet distribution construction process for the special case.
Recall that $Dir(\boldsymbol{\alpha})$ is the distribution of an $n$-dimensional simplex normalising the $n$-dimensional vector where each element is drawn from a Gamma distribution $Gamma(\alpha, \beta)$.
Setting $\alpha = 1$, $Gamma(\alpha, \beta)$ degenerates to exponential distribution $Exp(\beta)$.
For $n$-dimensional vector $(S_1, \cdots, S_n)$, where each element is drawn from $Exp(\beta)$, the normalized simplex $(\frac{S_1}{S}, \cdots, \frac{S_n}{S})$, where $S = \sum_{i\in N}S_i$, follows the Dirichlet distribution $Dir(1, \cdots, 1)$.
$Dir(1, \cdots, 1)$ is exactly the uniform distribution over the space of $n$-dimensional simplex.
\end{proof}

\subsubsection{Measuring Power Imbalances}

We use $\B_i^\theta (\X)$ to denote the random variable corresponding to the Banzhaf index of voter $i$ based on the random  weight profile $\X$, and given the WQR with quota $\theta$. We are interested in developing an analysis of the extent of deviations from perfectly balanced power (Theorem \ref{thm:no_wqr}) that can be expected under the probabilistic assumptions detailed above. We do that by using two main metrics: the variance that a single agent should expect with respect to their power index (single-agent variance); and the variance that one should expect to observe in the distribution of all power indices in the group of users (within-vector variance). The first type of variance will be studied analytically in the remaining of this section, while the second one will be the focus of our experiments in the last section.


\begin{definition}[Single-agent power variance]\label{def:single_agent_var}
Assume that the $n$-dimensional simplex variable $\X$ follows the Dirichlet distribution $P_\X = Dir(\boldsymbol{\alpha})$.
Fixing a quota $\theta$, for an arbitrary voter $i\in N$, we consider the random variable $\frac{\B_i^\theta (\X)}{X_1}$.
Then, we call $Var_{P_\X}[\frac{\B_i^\theta (\X)}{X_1}]$ the {\em single-agent power variance} of voter $i$.
\end{definition}
By our distributional assumptions, each voter is symmetric in the random process. We will therefore work, without loss of generality, with the single-agent variance of voter $1$, i.e., $Var_{P_\X}[\frac{\B_i^\theta (\X)}{X_1}]$.

\begin{definition}[Within-vector power variance]\label{def:within_vector_var}
Assume that the $n$-dimensional simplex variable $\X$ follows the Dirichlet distribution $P_\X = Dir(\boldsymbol{\alpha})$.
Fixing a quota $\theta$, we obtain the $n$-dimensional power-stake ratio vectors $\R_{\x\sim P_\X}(\theta) = (\frac{\B_1^\theta (\x)}{x_1}, \cdots, \frac{\B_n^\theta (\x)}{x_n})$.
Then, we refer to the expectation $\mathbb{E}_{\x\sim P_\X}[Var(\R_\x(\theta))]$ the {\em within-vector power variance}.
\end{definition}
Notice that a smaller within-vector variance indicates that elements in each power-stake ratio vector are closer to a constant across all weight profiles following $P_\X$, achieving $0$ when perfect power balance holds.
Observe furthermore that a small single-agent variance implies a small within-vector variance, but not vice versa.

\begin{example}
We illustrate the intuition of the within-vector variance and single-agent variance by a 5-agent example, with 5 weight profiles sampled from the Dirichlet distribution $Dir(1,1,1,1,1)$, i.e., $\x^1 = (0.32, 0.32, 0.31, 0.04, 0.01)$, $\x^2 = (0.20, 0.20, 0.44, 0.01, 0.15)$, $\x^3 = (0.02, 0.24, 0.38, 0.09, 0.27)$, $\x^4 = (0.08, 0.05, 0.17, 0.48, 0.22)$, and $\x^5 = (0.24, 0.04, 0.35, 0.24, 0.13)$.
Assume that the weight profile distribution is the uniform distribution supported by the above 5 weight profiles.
By computing the Banzhaf indices with $\theta = 0.5$, we obtain the power-stake ratio profile for each weight profile: $\R_{\x^1}(0.5) = (1.55, 1.55, 1.60, 0, 0)$, $\R_{\x^2}(0.5) = (1.22, 1.26, 1.71, 0, 1.70)$, $\R_{\x^3}(0.5) = (0, 2.07, 1.30, 0, 1.83)$, $\R_{\x^4}(0.5) = (1.52, 2.56, 0.74, 1.82, 0.57)$, and $\R_{\x^5}(0.5) = (1.56, 2.90, 1.76, 1.59, 0.99)$.

To compute the within-vector variance, we first compute the variance of each of the ratio profiles.
For instance, the variance of $\R_{\x^1}(0.5)$ is:
$$
Var(\R_{\x^1}(0.5)) = \frac{1}{5}[(1.55-0.94)^2+(1.55-0.94)^2+(1.60-0.94)^2+(-0.94)^2+(-0.94)^2] = 0.60.
$$
Then, the within-vector variance is the mean of the 5 ratio profiles' variances: 0.54.
To compute the single-agent variance, we consider agent $1$.
Her power-stake ratios across the 5 ratio profiles are $(1.55, 1.22, 0, 1.52, 1.56)$.
Therefore, her single-agent variance is the variance of this vector: $0.36$.

\end{example}

\subsubsection{Deriving the single-agent variance}\label{sec:single_variance}

We are now ready to state our main theoretical result, providing expressions for the exact computation of single-agent power variance.

\begin{theorem}\label{thm:single_variance}
Assume that the weight profile $\X$ follows the Dirichlet distribution $P_\X = Dir(\boldsymbol{\alpha})$. Then:
\begin{enumerate}
    \item when $0\le \theta \le 0.5$,
\begin{align}
\begin{split}
& Var_{P_\X}\left(\frac{\B_1^\theta (\X)}{X_1}\right) = \int_{c = 0}^\theta P_{X_1}(c)\frac{(\sum_{k=1}^{n-2}\frac{{n-1\choose k}}{2^{n-1}}(\Delta(\frac{\theta}{1-c}; k, \alpha) - \Delta(\frac{\theta-c}{1-c}; k, \alpha)))^2}{c^2}dc \\
& + \int_{c = \theta}^{1-\theta}P_{X_1}(c)\frac{(\sum_{k=1}^{n-2}\frac{{n-1\choose k}}{2^{n-1}}(\Delta(\frac{\theta}{1-c}; k, \alpha) - \Delta(\frac{\theta-c}{1-c}; k, \alpha))+\frac{1}{2^{n-1}})^2}{c^2}dc\\
& + \int_{c=1-\theta}^1 P_{X_1}(c)\frac{1}{c^2}dc - (\int_{c = 0}^\theta P_{X_1}(c)\frac{\sum_{k=1}^{n-2}\frac{{n-1\choose k}}{2^{n-1}}(\Delta(\frac{\theta}{1-c}; k, \alpha) - \Delta(\frac{\theta-c}{1-c}; k, \alpha))}{c}dc \\
& + \int_{c = \theta}^{1-\theta}P_{X_1}(c)\frac{\sum_{k=1}^{n-2}\frac{{n-1\choose k}}{2^{n-1}}(\Delta(\frac{\theta}{1-c}; k, \alpha) - \Delta(\frac{\theta-c}{1-c}; k, \alpha))+\frac{1}{2^{n-1}}}{c}dc + \int_{c=1-\theta}^1 P_{X_1}(c)\frac{1}{c}dc)^2,
\end{split}
\end{align}
\item when $0.5 < \theta\le 1$,
\begin{align}
\begin{split}
& Var_{P_\X}\left(\frac{\B_1^\theta (\X)}{X_1}\right) = \int_{c = 0}^{1-\theta} P_{X_1}(c)\frac{(\sum_{k=1}^{n-2}\frac{{n-1\choose k}}{2^{n-1}}(\Delta(\frac{\theta}{1-c}; k, \alpha) - \Delta(\frac{\theta-c}{1-c}; k, \alpha)))^2}{c^2}dc \\
& + \int_{c = 1-\theta}^{\theta}P_{X_1}(c)\frac{(\sum_{k=1}^{n-2}\frac{{n-1\choose k}}{2^{n-1}}(1 - \Delta(\frac{\theta-c}{1-c}; k, \alpha))+\frac{1}{2^{n-1}})^2}{c^2}dc + \int_{c=\theta}^1 P_{X_1}(c)\frac{1}{c^2}dc\\
& - (\int_{c = 0}^{1-\theta} P_{X_1}(c)\frac{\sum_{k=1}^{n-2}\frac{{n-1\choose k}}{2^{n-1}}(\Delta(\frac{\theta}{1-c}; k, \alpha) - \Delta(\frac{\theta-c}{1-c}; k, \alpha))}{c}dc \\
& + \int_{c = 1-\theta}^{\theta}P_{X_1}(c)\frac{\sum_{k=1}^{n-2}\frac{{n-1\choose k}}{2^{n-1}}(1 - \Delta(\frac{\theta-c}{1-c}; k, \alpha))+\frac{1}{2^{n-1}}}{c}dc + \int_{c=\theta}^1 P_{X_1}(c)\frac{1}{c}dc)^2,
\end{split}
\end{align}
\end{enumerate}
where $P_{X_1} (c)$ (see details of derivation in Lemma \ref{lem:c_in_dir}) is the probability that $X_1 = c$ given $\X$ follows $Dir(\boldsymbol{\alpha})$, and $\Delta(h; k,\alpha) = \frac{Bt(h; k\alpha, (n-1-k)\alpha)}{Bt(k\alpha, (n-1-k)\alpha)}$ with $Bt(\alpha, \beta)) = \int_0^1 t^{\alpha-1}(1-t)^{\beta-1}dt$ being the Beta function and $Bt(h;\alpha, \beta) = \int_0^h t^{\alpha-1}(1-t)^{\beta-1}dt$ being the incomplete Beta function.
\end{theorem}


Before providing the detailed proof of Theorem \ref{thm:single_variance}, we outline the structure of our argument. First of all, we observe that the single-agent variance equals $\mathbb{E}_{P_\X}[(\B_1^\theta (\X)/X_1)^2] - [\mathbb{E}_{P_\X}(\B_1^\theta (\X)/X_1)]^2$.
We show how to derive the expression for $\mathbb{E}_{P_\X}(\B_1^\theta (\X)/X_1)$. The expression for $\mathbb{E}_{P_\X}[(\B_1^\theta (\X)/X_1)^2]$ can be obtained similarly.

To derive the expected power-stake ratio of agent $1$, i.e., $\mathbb{E}_{P_\X}(\B_1^\theta (\X)/X_1)$, we need to determine the pivotality of the agent under the assumptions that weight profiles are distributed according to a Dirichlet distribution and that coalitions form uniformly at random (the assumption underlying the definition of the Banzhaf index).
Specifically, given that agent $1$'s weight is fixed as $c$ (with probability $P_{X_1}(c)$, Lemma \ref{lem:c_in_dir}), we compute the density of the weight of a random $k$-agent coalition drawn from $N\setminus\{1\}$ uniformly (the probability of drawing a $k$-agent coalition is ${n-1 \choose k}/2^{n-1}$), aligning with the definition of the Banzhaf index.
Lemma \ref{lem:X-1} shows that the coalition weight follows a Beta distribution.
The Beta distribution's density allows us to compute the probability that the coalition's weight is below the quota $\theta$ but it exceeds $\theta$ when agent $1$ joins the coalition, i.e., the pivotality probability of agent $1$ (Equation \ref{eq:pivotal_1}).
Then, the mean of the ratio between this pivotality probability and agent $1$'s weight $c$, over all possible coalition sizes $k$ and agent $1$'s weight $c$, is the expected power-stake ratio of agent $1$, i.e., $\mathbb{E}_{P_\X}(\B_1^\theta (\X)/X_1)$.
We follow a similar process to compute $\mathbb{E}_{P_\X}[(\B_1^\theta (\X)/X_1)^2]$, and obtain $Var_{P_\X}(\B_1^\theta (\X)/X_1)$. We now proceed to state and prove the two auxiliary results needed for the proof of Theorem \ref{thm:single_variance}: Lemma \ref{lem:c_in_dir} and Proposition \ref{prop:banzhaf_gamma} (which in turn relies on Lemma \ref{lem:X-1}).



\begin{lemma}\label{lem:c_in_dir}
Given that $\X$ follows $Dir(\boldsymbol{\alpha})$, the probability that $X_1 = c$ is

\begin{equation}
P_{X_1}(c) = \frac{\Gamma (n\alpha)}{\Gamma (\alpha)\Gamma ((n-1)\alpha)}c^{\alpha - 1}(1-c)^{(n-1)\alpha -1}.
\end{equation}
\end{lemma}
\begin{proof}
Since $\X$ follows $Dir(\boldsymbol{\alpha})$, we have that $X_1$ follows the Beta distribution $Beta(\alpha, (n-1)\alpha)$, that is: 
$
P_{X_1}(c) = \frac{\Gamma (n\alpha)}{\Gamma (\alpha)\Gamma ((n-1)\alpha)}c^{\alpha - 1}(1-c)^{(n-1)\alpha -1}.
$
\end{proof}

\begin{proposition}\label{prop:banzhaf_gamma}
Assume that $\X$ follows $Dir(\boldsymbol{\alpha})$.
Conditioned on $X_1 = c$ (where $0<c<1$), the expected Banzhaf index of agent $1$, denoted as $\mathbb{E}(\B^\theta_{1,c})$, is as follows:
\begin{enumerate}
    \item if $0\le \theta \le 0.5$,
\begin{align*}
\B^\theta_{1,c} = \left\{\begin{array}{lrc}
\sum_{k=1}^{n-2}\frac{{n-1\choose k}}{2^{n-1}}(\Delta(\frac{\theta}{1-c};k) - \Delta(\frac{\theta-c}{1-c};k)) & \text{ for } 0<c\le \theta, \\
\sum_{k=1}^{n-2}(\Delta(\frac{\theta}{1-c};k) - \Delta(\frac{\theta-c}{1-c};k))+\frac{1}{2^{n-1}} & \text{ for } \theta < c < 1-\theta, \\ 
1 & \text{ for } 1-\theta \le c <1. \\ 
\end{array}\right.
\end{align*}

\item if $0.5 < \theta \le 1$,
\begin{align*}
\B^\theta_{1,c} = \left\{\begin{array}{lrc}
\sum_{k=1}^{n-2}\frac{{n-1\choose k}}{2^{n-1}}(\Delta(\frac{\theta}{1-c};k) - \Delta(\frac{\theta-c}{1-c};k)) & \text{ for } 0<c\le \theta, \\
\sum_{k=1}^{n-2}\frac{{n-1\choose k}}{2^{n-1}}(1-\Delta(\frac{\theta-c}{1-c};k))+\frac{1}{2^{n-1}} & \text{ for } 1-\theta \le c\le \theta,\\
1 & \text{ for } \theta< c < 1
\end{array}\right.
\end{align*}
\end{enumerate}
\end{proposition}

\begin{proof}
Given that $\X$ follows $Dir(\boldsymbol{\alpha})$, to compute the expected Banzhaf index of agent $1$, we need to compute the probability that agent $1$ is a swing agent for a random coalition drawn from the uniform distribution.
Especially, for a randomly drawn coalition $N'\subseteq N\setminus \{1\}$, we compute the probability of the event where $X_{N'}\le \theta$ and $X_{N'\cup\{1\}} > \theta$, conditioned on $X_1 = c$.
Therefore, to compute the probability, we need the density of $X_{N'}$, which follows a Beta distribution since the weight profile $\X_{-1} = (X_2, \cdots, X_n)$ follows a rescaled Dirichlet distribution conditioned on $X_1=c$. This is shown in the following auxiliary lemma.
\begin{lemma}\label{lem:X-1}
$\frac{1}{1-c}\X_{-1}$ follows the Dirichlet distribution $Dir(\boldsymbol{\alpha})$.
\end{lemma}
\begin{proof}
Recall the Dirichlet process based on the Gamma distribution: (1) we have multiple mutually independent variables following identical scale Gamma distributions; and (2) their normalized multi-variate vector follows a Dirichlet distribution parameterized by the shape parameters of the Gamma distributions.
We prove the lemma following the process.

Let $\S = (S_1, \cdots, S_n)$ be the random variables denoting the stakes of each agent, each drawn from an identical Gamma distribution $\Gamma (\alpha)$ independently.
Let $S_{-1} = \sum_{i\in N\setminus \{1\}}S_i$ and $S=\sum_{i\in N}S_i$.
Since we assume $X_1 = c$ in weight profile $\X$, it holds that $S_1 = cS$.
Then, for each $i\in N\setminus \{1\}$, we have:
\begin{align}\label{eq:lem:X-1}\frac{X_i}{1-c} = \frac{S_i/S}{1-S_1/S} = \frac{S_i}{S-S_1} = \frac{S_i}{S_{-1}}.
\end{align}
Let $\mathbf{V} = (V_2 = \frac{X_2}{1-c}, \cdots, V_n = \frac{X_n}{1-c})$.
Equation \ref{eq:lem:X-1} is exactly the Dirichlet distribution construction process of $\mathbf{V} = \frac{1}{1-c}\X_{-1}$ from $\S_{-1}$.
Therefore, we have that:
$$
\frac{1}{1-c}\X_{-1} = \mathbf{V} \sim Dir(\boldsymbol{\alpha}).
$$
This completes the proof of Lemma \ref{lem:X-1}.
\end{proof}



With the lemma in place, we can then proceed to complete the proof of Proposition \ref{prop:banzhaf_gamma}.
Since the weight profile of $N_{-1}$ follows the rescaled Dirichlet distribution given $X_1=c$, and the stakes of each agent in $N_{-1}$ are drawn from $\Gamma(\alpha)$ i.i.d. (i.e., all agents' roles are symmetric), we have that, for a uniformly random coalition $N'\subseteq N_{-1}$ with size $k$ ($0\le k\le n-1$), its weight follows Beta distribution $Beta(k\alpha, (n-1-k)\alpha)$ (Chapter 25, \cite{johnson1995continuous})
Now, we consider, for each specific $k$, what the probability is that agent $1$ swings for a uniformly random coalition with size $k$.
The pivotality probability of agent $1$ for each specific $k$ equals this probability multiplied by the probability of uniformly drawing a $k$-agent coalition.

Especially, the probability of uniformly drawing a $k$-agent coalition $N_k$ from $N_{-1}$ is:
$$P(N_k) = {n-1\choose k}/2^{n-1}.$$

We derive the probability that agent $1$ swings for a random $k$-agent coalition , i.e., the probability of the event where $X_{N_k}\le \theta$ and $X_{N_k}+c >\theta$, and then, the expected Banzhaf index of agent $1$ is the expectation over all values of $k$.
We derive the probability in two exhaustive cases: (1) $\theta\in [0,0.5]$, and (2) $\theta\in (0.5, 1]$.
Each of the two cases is divided into three subcases, to which we turn now.



\begin{enumerate}

\item[(1.1)] {\bf [$\boldsymbol{\theta\in [0,0.5]}$ and $\boldsymbol{0<c\le \theta}$]}.
In this case, agent $1$ can neither swing for the empty set ($k=0$), nor the unanimous set ($k=n-1$), since $\theta\le 0.5$ and $c \le \theta$.
Therefore, if $k = 0$ or $n-1$, the probability that agent $1$ is a swing agent is $0$.
For $1\le k \le n-2$, the probability that agent $1$ swings for $N_k$ equals the probability that under condition $X_1 = c$, $N_k$'s weight $X_{N_k}$ is between $\theta-c$ and $\theta$, i.e., we denote it as:
\begin{align}\label{eq:pivotal_1}P_{1,N_k} = P(\theta-c < X_{N_k} \le \theta).\end{align}
By Lemma \ref{lem:X-1}, $\X_{-1}$ follows a rescaled Dirichlet distribution, from which we conclude that $\frac{1}{1-c}X_{N_{k}}$ follows a Beta distribution $Beta(k\alpha, (n-1-k)\alpha)$ and its accumulative density is:
$$P(\frac{X_{N_k}}{1-c}\le h) = \frac{Bt(h;k\alpha, (n-1-k)\alpha)}{Bt(k\alpha, (n-1-k)\alpha))} = \Delta(h;k, \alpha).$$
Therefore, we have that
\begin{align}\label{eq:low_q_1}
P_{1,N_k} = P(\frac{\theta-c}{1-c} < \frac{X_{N_k}}{1-c} \le \frac{\theta}{1-c})
=  \Delta(\frac{\theta}{1-c};k, \alpha) - \Delta(\frac{\theta-c}{1-c};k, \alpha).
\end{align}
Then, under condition $X_1=c$, across all values of $k$, we have the expected Banzhaf of agent $1$ is:
\begin{align}
\B^\theta_{1,c} = \sum_{k=1}^{n-2}\frac{{n-1\choose k}}{2^{n-1}}P_{1,N_k} = \sum_{k=1}^{n-2}\frac{{n-1\choose k}}{2^{n-1}}(\Delta(\frac{\theta}{1-c};k, \alpha) - \Delta(\frac{\theta-c}{1-c};k, \alpha)).
\end{align}

\item[(1.2)] {\bf [$\boldsymbol{\theta\in [0,0.5]}$ and $\boldsymbol{\theta < c < 1-\theta}$]}.
In this case, agent $1$ cannot swing for $N_{-1}$ (i.e., $k=n-1$), since that $c\le 1-\theta$ indicates $X_{N_{-1}} > \theta$.
However, agent $1$ must be a swing agent for the empty set (i.e., $k=0$), since $c> \theta$.

For $1\le k \le n-2$, the probability that agent $1$ swings for $N_k$ equals Equation \ref{eq:low_q_1}.
Therefore, we have that the expected Banzhaf index of agent $1$ in this case is:
\begin{align}
\B_{1,c}^\theta = \sum_{k=1}^{n-2}\frac{{n-1\choose k}}{2^{n-1}}P_{1,N_k} = \sum_{k=1}^{n-2}\frac{{n-1\choose k}}{2^{n-1}}(\Delta(\frac{\theta}{1-c};k, \alpha) - \Delta(\frac{\theta-c}{1-c};k, \alpha)) + \frac{1}{2^{n-1}}.
\end{align}

\item[(1.3)] {\bf [$\boldsymbol{\theta\in [0,0.5]}$ and $\boldsymbol{1-\theta \le  c <1}$]}.
Since $c\ge 1-\theta$, we have that for any $N'\subseteq N_{-1}$, $X_{N'}\le \theta$ holds.
Hence, agent $1$ is a dictator, i.e., she swings for each coalition $N'\subseteq N_{-1}$.
Then, her expected Banzhaf index is:
\begin{align}
\B_{1,c}^\theta = 1.
\end{align}

\item[(2.1)] {\bf [$\boldsymbol{\theta\in (0.5, 1]}$ and $\boldsymbol{0< c < 1-\theta}$]}.
In this case, agent $1$ cannot swing for either the empty set ($k=0$) or the unanimous set ($k=n-1$), since $c < \theta$ and $X_{N_{-1}}> \theta$.
For $2\le k \le n-2$, we have that the probability that agent $1$ swings for $N_k$ has the same expression as Equation \ref{eq:low_q_1}.
Therefore, we have that in this case, the expected Banzhaf index of agent $1$ is:
\begin{align}
\B_{1,c}^\theta = \sum_{k=1}^{n-2}\frac{{n-1\choose k}}{2^{n-1}}P_{1,N_k} = \sum_{k=1}^{n-2}\frac{{n-1\choose k}}{2^{n-1}}(\Delta(\frac{\theta}{1-c};k, \alpha) - \Delta(\frac{\theta-c}{1-c};k, \alpha))
\end{align}

\item[(2.2)] {\bf [$\boldsymbol{\theta\in (0.5, 1]}$ and $\boldsymbol{1-\theta \le c \le \theta}$]}.
Since $c \ge 1-\theta$, we have that $X_{N_{-1}}\le \theta$, and therefore, each winning coalition (i.e., a coalition with a weight exceeding the quota) must contain agent $1$.
However, agent $1$ cannot swing for the empty set ($k=0$) since $c\le \theta$, and she must be a swing agent for $N_k$ when $k=n-1$.
Hence, the probability that agent $1$ swings for $N_k$ ($k\ge 1$) is the probability that $N_k$ has a weight larger than $\theta-c$.
That is:
\begin{align}
P_{1, N_k} = P(X_{N_k} > \theta-c)
= 1-\Delta(\frac{\theta - c}{1-c}; k, \alpha).
\end{align}
Subsequently, we have that the expected Banzhaf index of agent $1$ is:
\begin{align}
\B_{1,c}^\theta = \sum_{k=1}^{n-2}\frac{{n-1\choose k}}{2^{n-1}}(1-\Delta(\frac{\theta - c}{1-c}; k, \alpha)) + \frac{1}{2^{n-1}}.
\end{align}

\item[(2.3)] {\bf [$\boldsymbol{\theta \in (0.5, 1]}$ and $\boldsymbol{\theta < c < 1}$]}.
Agent $1$ is a dictator in this case, and therefore, her expected Banzhaf is:
\begin{align}
\B_{1,c}^\theta = 1.
\end{align}
\end{enumerate}
This completes the proof of Proposition \ref{prop:banzhaf_gamma}. 
\end{proof}
With Proposition \ref{prop:banzhaf_gamma} we now have all the scaffolding in place to prove Theorem \ref{thm:single_variance}.
\begin{proof}[Proof of Theorem \ref{thm:single_variance}]
We first derive the expectation of $\mathbb{E}(\frac{\B^\theta_1}{X_1})$ as the integral of $\mathbb{E}(\frac{\B^\theta_{1,c}}{c})$ across $c\in [0,1]$, where $\B^\theta_{1,c}$ is obtained in Proposition \ref{prop:banzhaf_gamma}.
In a similar manner we derive the expectation of $\mathbb{E}[(\frac{\B^\theta_1}{X_1})^2]$.
With these expressions in place, the single-agent variance can also be obtained. We still derive $\mathbb{E}(\frac{\B^\theta_1}{X_1})$ in two cases: (1) $0\le \theta\le 0.5$, and (2) $0.5<\theta\le 1$.

\begin{enumerate}

\item $\boldsymbol{0\le \theta\le 0.5}$.
The expected power-stake ratio of agent $1$, denoted as $\mathbb{E}(\frac{\B^\theta_1}{X_1})$, is the expectation of $\mathbb{E}(\frac{\B^\theta_{1,c}}{X_1})$ across all values of $c$:
\begin{align}\label{eq:low_quota_expected_ratio}
\begin{split}
\mathbb{E}(\frac{\B^\theta_1}{X_1}) & = \int_{c = 0}^\theta P_{X_1}(c)\frac{\sum_{k=1}^{n-2}\frac{{n-1\choose k}}{2^{n-1}}(\Delta(\frac{\theta}{1-c}; k, \alpha) - \Delta(\frac{\theta-c}{1-c}; k, \alpha))}{c}dc \\
& + \int_{c = \theta}^{1-\theta}P_{X_1}(c)\frac{\sum_{k=1}^{n-2}\frac{{n-1\choose k}}{2^{n-1}}(\Delta(\frac{\theta}{1-c}; k, \alpha) - \Delta(\frac{\theta-c}{1-c}; k, \alpha))+\frac{1}{2^{n-1}}}{c}dc\\
& + \int_{c=1-\theta}^1 P_{X_1}(c)\frac{1}{c}dc,
\end{split}
\end{align}
where $P_{X_1}(c)$ is the probability that agent $1$'s weight is $c$ given $\X$ follows the Dirichlet distribution $Dir(\boldsymbol{\alpha})$, provided in Lemma \ref{lem:c_in_dir}.

\item $\boldsymbol{0.5 < \theta \le 1}$.
Similarly, we have the expected power-stake ratio of agent $1$ as follows:
\begin{align}\label{eq:high_quota_expected_ratio}
\begin{split}
\mathbb{E}(\frac{\B^\theta_1}{X_1}) & = \int_{c = 0}^{1-\theta} P_{X_1}(c)\frac{\sum_{k=1}^{n-2}\frac{{n-1\choose k}}{2^{n-1}}(\Delta(\frac{\theta}{1-c}; k, \alpha) - \Delta(\frac{\theta-c}{1-c}; k, \alpha))}{c}dc \\
& + \int_{c = 1-\theta}^{\theta}P_{X_1}(c)\frac{\sum_{k=1}^{n-2}\frac{{n-1\choose k}}{2^{n-1}}(1 - \Delta(\frac{\theta-c}{1-c}; k, \alpha))+\frac{1}{2^{n-1}}}{c}dc\\
& + \int_{c=\theta}^1 P_{X_1}(c)\frac{1}{c}dc.
\end{split}
\end{align}
\end{enumerate}

\noindent 
By a similar strategy, we obtain the expression of $\mathbb{E}[(\frac{\B^\theta_1}{X_1})^2]$.
Subsequently, we derive the variance of $\frac{\B^\theta_1}{X_1}$ by $var(\frac{\B^\theta_1}{X_1}) = \mathbb{E}[(\frac{\B^\theta_1}{X_1})^2] - [\mathbb{E}(\frac{\B^\theta_1}{X_1})]^2$, i.e.,
\begin{enumerate}
\item $\boldsymbol{0\le \theta \le 0.5}$
\begin{align}\label{eq:low_quota_variance}
\begin{split}
& var(\frac{\B^\theta_1}{X_1}) = \int_{c = 0}^\theta P_{X_1}(c)\frac{(\sum_{k=1}^{n-2}\frac{{n-1\choose k}}{2^{n-1}}(\Delta(\frac{\theta}{1-c}; k, \alpha) - \Delta(\frac{\theta-c}{1-c}; k, \alpha)))^2}{c^2}dc \\
& + \int_{c = \theta}^{1-\theta}P_{X_1}(c)\frac{(\sum_{k=1}^{n-2}\frac{{n-1\choose k}}{2^{n-1}}(\Delta(\frac{\theta}{1-c}; k, \alpha) - \Delta(\frac{\theta-c}{1-c}; k, \alpha))+\frac{1}{2^{n-1}})^2}{c^2}dc + \int_{c=1-\theta}^1 P_{X_1}(c)\frac{1}{c^2}dc\\
& - (\int_{c = 0}^\theta P_{X_1}(c)\frac{\sum_{k=1}^{n-2}\frac{{n-1\choose k}}{2^{n-1}}(\Delta(\frac{\theta}{1-c}; k, \alpha) - \Delta(\frac{\theta-c}{1-c}; k, \alpha))}{c}dc \\
& + \int_{c = \theta}^{1-\theta}P_{X_1}(c)\frac{\sum_{k=1}^{n-2}\frac{{n-1\choose k}}{2^{n-1}}(\Delta(\frac{\theta}{1-c}; k, \alpha) - \Delta(\frac{\theta-c}{1-c}; k, \alpha))+\frac{1}{2^{n-1}}}{c}dc + \int_{c=1-\theta}^1 P_{X_1}(c)\frac{1}{c}dc)^2.
\end{split}
\end{align}
\item $\boldsymbol{0.5<\theta\le 1}$
\begin{align}\label{eq:high_quota_variance}
\begin{split}
& var(\frac{\B^\theta_1}{X_1}) = 
\int_{c = 0}^{1-\theta} P_{X_1}(c)\frac{(\sum_{k=1}^{n-2}\frac{{n-1\choose k}}{2^{n-1}}(\Delta(\frac{\theta}{1-c}; k, \alpha) - \Delta(\frac{\theta-c}{1-c}; k, \alpha)))^2}{c^2}dc \\
& + \int_{c = 1-\theta}^{\theta}P_{X_1}(c)\frac{(\sum_{k=1}^{n-2}\frac{{n-1\choose k}}{2^{n-1}}(1 - \Delta(\frac{\theta-c}{1-c}; k, \alpha))+\frac{1}{2^{n-1}})^2}{c^2}dc + \int_{c=\theta}^1 P_{X_1}(c)\frac{1}{c^2}dc\\
& - (\int_{c = 0}^{1-\theta} P_{X_1}(c)\frac{\sum_{k=1}^{n-2}\frac{{n-1\choose k}}{2^{n-1}}(\Delta(\frac{\theta}{1-c}; k, \alpha) - \Delta(\frac{\theta-c}{1-c}; k, \alpha))}{c}dc \\
& + \int_{c = 1-\theta}^{\theta}P_{X_1}(c)\frac{\sum_{k=1}^{n-2}\frac{{n-1\choose k}}{2^{n-1}}(1 - \Delta(\frac{\theta-c}{1-c}; k, \alpha))+\frac{1}{2^{n-1}}}{c}dc + \int_{c=\theta}^1 P_{X_1}(c)\frac{1}{c}dc)^2.
\end{split}
\end{align}
\end{enumerate}
This completes the proof of Theorem \ref{thm:single_variance}.
\end{proof}

\subsubsection{Numerical analysis of the single-agent variance}\label{sec:numerical_analysis}
\begin{figure}[t]
\centering
\begin{subfigure}[t]{0.49\linewidth}
\centering
\includegraphics[width=\linewidth]{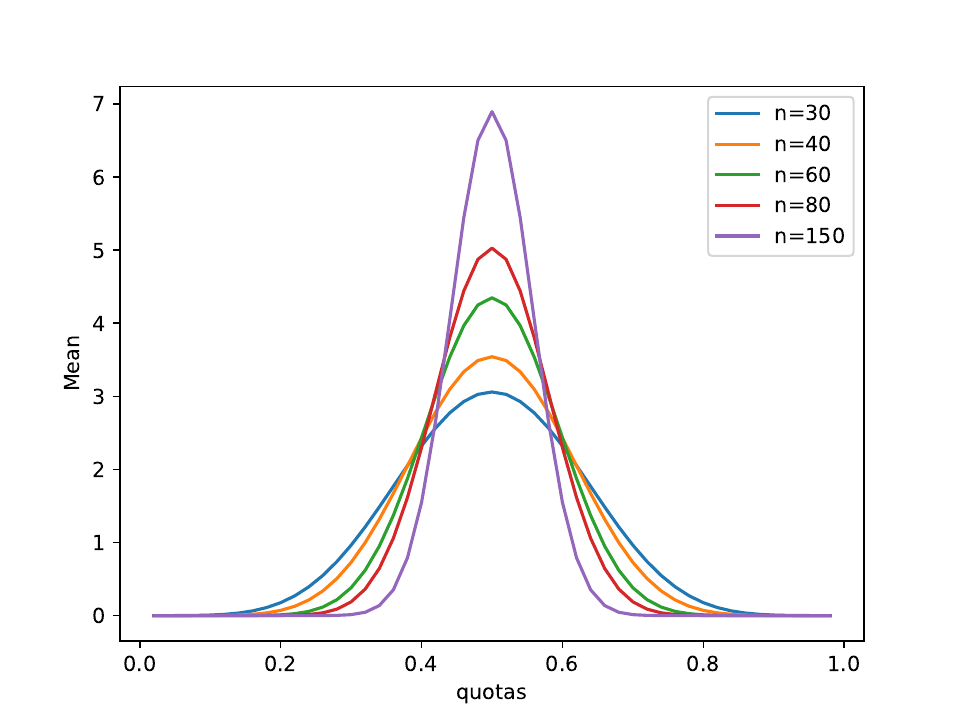}
\caption{Mean of power-stake ratio: $\alpha = 1$}
\label{fig:analytical_mean_alpha1}
\end{subfigure}\hfill
\begin{subfigure}[t]{0.49\linewidth}
\centering
\includegraphics[width=\linewidth]{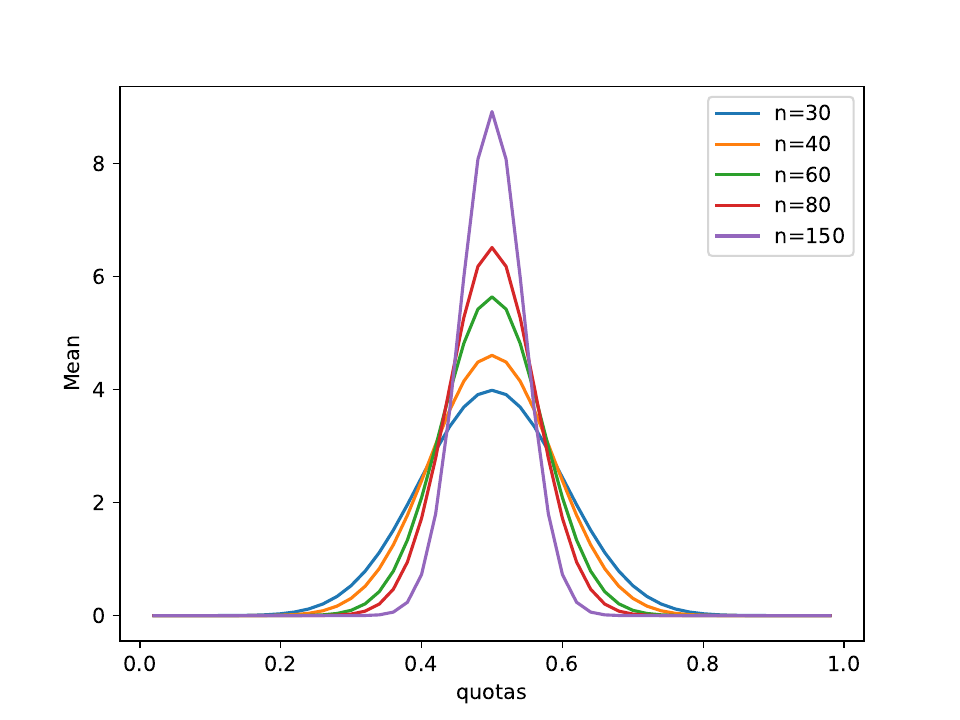}
\caption{Mean of power-stake ratio: $\alpha = 5$}
\label{fig:analytical_mean_alpha5}
\end{subfigure}
\begin{subfigure}[t]{0.49\linewidth}
\centering
\includegraphics[width=\linewidth]{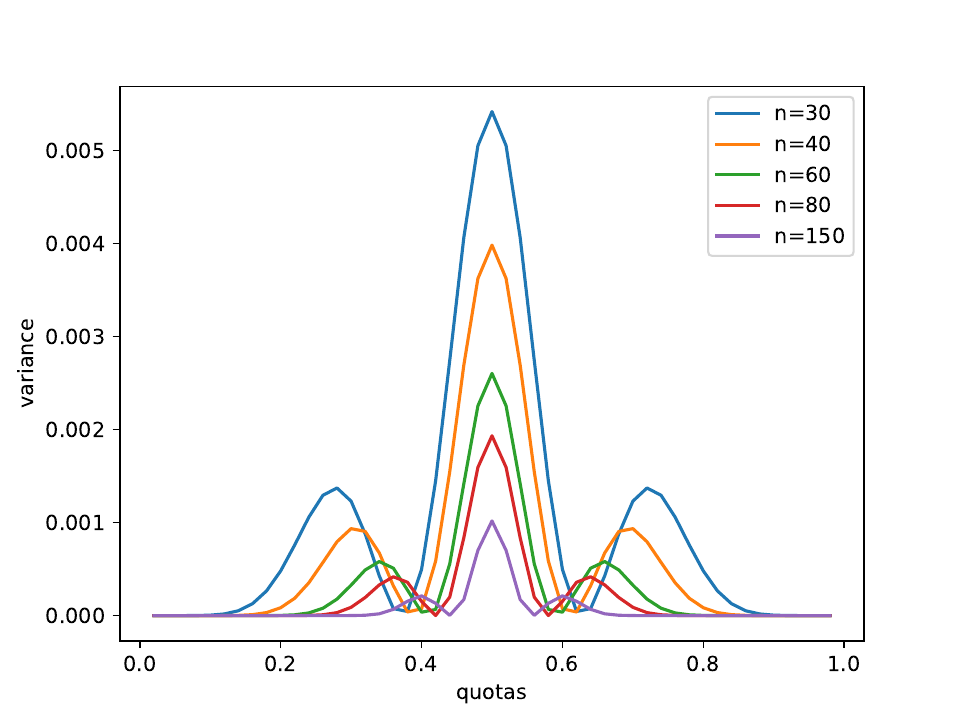}
\caption{Single-agent variance: $\alpha = 1$}
\label{fig:analytical_variance_alpha1}
\end{subfigure}\hfill
\begin{subfigure}[t]{0.49\linewidth}
\centering
\includegraphics[width=\linewidth]{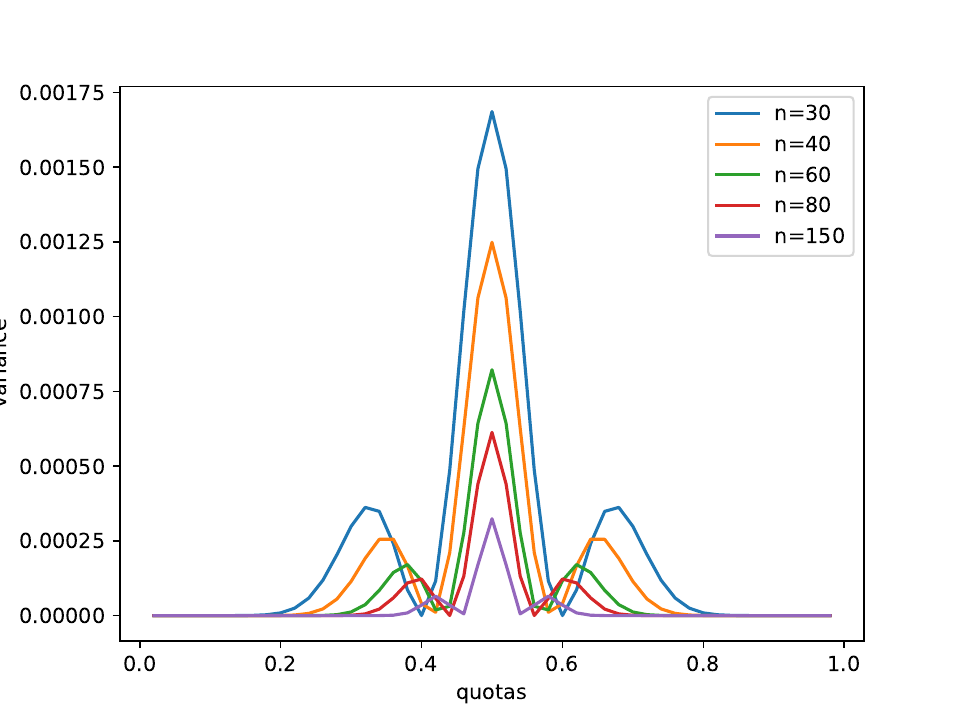}
\caption{Single-agent variance: $\alpha = 5$}
\label{fig:analytical_variance_alpha5}
\end{subfigure}
\caption{Means of the power-stake ratios and single-agent variances of agent $1$ by varying $\theta$ and $n$.}
\label{fig:analytical_plots}
\end{figure}
We present the expectations (Equations \ref{eq:low_quota_expected_ratio} and \ref{eq:high_quota_expected_ratio}) and variances (Equations \ref{eq:low_quota_variance} and \ref{eq:high_quota_variance}) of agent $1$'s power-stake ratios by varying quota $\theta$ in range $(0,1)$ with step $0.01$ and varying $n$ in $\{30, 40, 60, 80, 150\}$, shown in Figure \ref{fig:analytical_plots}.
We select two values of $\alpha$, namely, $1$ and $5$, where $\alpha = 1$ is the exponential distribution, while $\alpha = 5$ forms a single peaked density function with a higher variance than $\alpha = 1$.\footnote{To validate our analysis, we use simulations to generate the corresponding expectations and variances of agent $1$'s power-stake ratios in Appendix \ref{sec:verify}. The consistent curves verify the correctness of the analytical results.}
Additional figures for $\alpha$ matching stake distribution observed in practice in Project Catalyst are presented in Section~\ref{sec:experiments}.

The most interesting observation is that for each $n$, the variance (Figures \ref{fig:analytical_variance_alpha1} and \ref{fig:analytical_variance_alpha5}) is closest to $0$ when the quota is around $0.6$ and $0.4$.
These results provide cues for the identification of optimal WQR based on the single-agent variance analysis.
This trend differs from that of the ratio expectations (Figures \ref{fig:analytical_mean_alpha1} and \ref{fig:analytical_mean_alpha5}): variances fluctuate, reach a lowest value around $0$ in both intervals $(0,0.5)$ and $(0.5, 1)$, however, the expectations are monotone in both intervals.
It shows the balance quotas are not a consequence of low expected ratio values.
As $n$ increases, the balance quotas move towards $0.5$, aligning with conjectures made in \cite{slomczynski2006penrose}.
Importantly, this indicates that in practice, the possibility of finding a quota that achieves a good power balance in expectation.

Observe that the expectations and variances reach the highest when $\theta = 0.5$.
With larger agent size $n$, the ratio's expectation increases, however, it has a lower variance.
This inverse trend reveals that though agents have higher power-stake ratios on average when $n$ becomes larger, their ratios tend to be more similar, achieving better power balance.

Note that for extreme quotas, i.e., quotas close to $0$ or $1$, near-$0$ variance does not indicate the presence of power balance. The near-$0$ variance is instead driven by uniformly low expected power values.
Intuitively, when $\theta$ is close to $1$, voting becomes by unanimity, and each agent's Banzhaf power equals $1/2^{n-1}$, which is considerably low. The variance then reduces to the variance of the inverse of agents' weights.

\subsubsection{When $\alpha = 1$ and $\theta = 0.5$: A special case and generalization previous results}\label{sec:special_case}
\citet{jelnov2014voting} consider similar measures of power balance, where they derive the expected ratio between agent $1$'s Banzhaf index and her weight, i.e., $\mathbb{E}(\frac{\B^\theta_1}{X_1})$, by assuming that $\X$ follows the uniform distribution and weighted majority rule (i.e., $\theta=0.5$).
We note that our model is a strict generalisation of \cite{jelnov2014voting}: the multi-variate uniform distribution is a special case of the Dirichlet distribution by setting $\alpha = 1$, under which the original Gamma distribution becomes the exponential distribution (Lemma \ref{lem:uniform_distribution}).
We further note that an unbounded expected power-stake ratio does not necessarily indicate the failure of power balance.
That is, bounding $\mathbb{E}(\frac{\B^\theta_1}{X_1})$ towards $1$ indicates a good power balance w.r.t. Equation \ref{eq:proportional}, but not vice versa.
However, a bounded within-vector variance is a sufficient and necessary condition for balanced power.

In a similar but more restricted streamline, \cite{jelnov2014voting} assumes that a weight profile, also denoted as an $n$-dimensional simplex, is drawn from a multi-variate uniform distribution.
To relate our work to \cite{jelnov2014voting}, we show that fixing the Dirichlet distribution to the relevant special case yields precisely the expression of the expected power-stake ratio from \cite{jelnov2014voting}, where the authors deduce it as:
\begin{align}\label{eq:mean_ratio_special}
\begin{split}
& \mathbb{E}(\frac{\B^{0.5}_1}{X_1}) = \frac{1}{2^{n-1}}\int_{0}^{0.5}(1-n)(1-c)^{n-2}\\
& \left[\sum_{k=1}^{n-2}{n-1\choose k}\sum_{j=k}^{n-2}{n-2\choose j}\left((\frac{0.5}{1-c})^j(1-\frac{0.5}{1-c})^{n-2-j} - (\frac{0.5-c}{1-c})^j(1-\frac{0.5-c}{1-c})^{n-2-j}\right)\right]\frac{1}{c}dc\\
& + \int_{0.5}^1 \frac{(1-n)(1-c)^{n-2}}{c}dc
\end{split}
\end{align}
Observe that Equation \ref{eq:mean_ratio_special} is a function of $n = |N|$, and \cite{jelnov2014voting} shows that it is unbounded and increases with $n$.


\begin{corollary}\label{cor:uniform}
For the Expected power-stake ratio of agent $1$, i.e., Equations \ref{eq:low_quota_expected_ratio} and \ref{eq:high_quota_expected_ratio}, setting $\alpha = 1$ and $\theta = 0.5$, it coincides with (Equation (19), Lemma 6 in \cite{jelnov2014voting}).
\end{corollary}
\begin{proof}
By setting $\alpha = 1$ and $\theta = 0.5$, the expected power=stake ratio of agent $1$ becomes:
\begin{align*}
\mathbb{E}(\frac{\B_1^{0.5}}{X_1}) = \int_{0}^{0.5}P_{X_1}(c)\frac{\sum_{k=1}^{n-2}\frac{{n-1\choose k}}{2^{n-1}}(\Delta(\frac{0.5}{1-c};k,1) - \Delta(\frac{0.5-c}{1-c};k,1))}{c}dc + \int_{0.5}^1P_{X_1}(c)\frac{1}{c}dc.
\end{align*}
Recall that $\Delta(\frac{0.5}{1-c};k,1)$ is the cumulative function of Beta distribution $Beta(k,1)$, denoted as the regularised incomplete beta function
$$\Delta(\frac{0.5}{1-c};k,1) = \frac{Bt(\frac{0.5}{1-c};k,n-1-k)}{Bt(k,n-1-k)}.$$
When the Beta distribution's parameters are positive integers, it cumulative distribution can be written as the cumulative function of a corresponding binomial distribution (26.5.24 in \cite{abramowitz1948handbook}), i.e., 
\begin{align*}\begin{split}\Delta(\frac{0.5}{1-c};k,1) = & \sum_{j = k}^{n-1-k+k-1}{n-1-k+k-1\choose j}(\frac{0.5}{1-c})^j(1-\frac{0.5}{1-c})^{n-1-k+k-1-j}\\ = & \sum_{j = k}^{n-2}{n-2\choose j}(\frac{0.5}{1-c})^j(1-\frac{0.5}{1-c})^{n-2-j}.
\end{split}
\end{align*}
Similarly, we have that
\begin{align*}
\Delta(\frac{0.5-c}{1-c};k,1)) = \sum_{j=k}^{n-2}{n-2\choose j}(\frac{0.5-c}{1-c})^j(1-\frac{0.5-c}{1-c})^{n-2-j}.
\end{align*}
Since $P_{X_1}(c) = \frac{\Gamma(n\alpha)}{\Gamma(\alpha)\Gamma((n-1)\alpha)}c^{\alpha - 1}(1-c)^{(n-1)\alpha-1}$, by setting $\alpha = 1$, we have that
\begin{align*}
P_{X_1}(c) = \frac{\Gamma(n)}{\Gamma(1)\Gamma(n-1)}(1-c)^{n-2} = \frac{(n-1)\Gamma(n-1)}{\Gamma(n-1)}(1-c)^{n-2} = (n-1)(1-c)^{n-2}.
\end{align*}
Therefore, we have that
\begin{align*}
\begin{split}
& \mathbb{E}(\frac{\B_1^{0.5}}{X_1}) = \int_{0}^{0.5}\frac{(n-1)(1-c)^{n-2}}{c2^{n-1}}\\ & \sum_{k=1}^{n-2}{n-2\choose k}\sum_{j=k}^{n-2}{n-2\choose j}\left(\frac{0.5}{1-c})^j(1-\frac{0.5}{1-c})^{n-2-j} - \frac{0.5-c}{1-c})^j(1-\frac{0.5-c}{1-c})^{n-2-j}\right)dc\\
& + \int_{0.5}^1\frac{(n-1)(1-c)^{n-2}}{c}dc,
\end{split}
\end{align*}
which is exactly Equation \ref{eq:mean_ratio_special}.
\end{proof}

\section{Experiments}
\label{sec:experiments}

We apply the framework for the analysis of power imbalances developed in the previous sections to simulated data, including data generated by fitting our probabilistic model on real-world election data from Project Catalyst.
Specifically, our experiments study how the normalized Banzhaf index compares to normalized
stake shares as a function of the quota and the underlying stake distribution. Our goal is to quantify both the average deviation from perfect power balance (Equation \ref{eq:proportional}) and how such deviations are distributed across agents (Definition \ref{def:within_vector_var}).
In essence, we develop an experimental approach based on real-world data in order to address the following question: which quota does a WQR need in order to yield a voting power distribution that is as close as possible to perfect balance?


\subsection{Experimental Setup}
\label{sec:experiments:setup}

\paragraph{Evaluation metrics.}
We measure the within-vector variances (Definition \ref{def:within_vector_var}) based on randomly drawn normalized weight profiles $\X$ for each parameter setting.
Notice that the within-vector variances are computed based on different weight profiles, and hence, there is inter-profile scale difference.
Especially, for two weight profiles $\X^1$ and $\X^2$, we have two corresponding Banzhaf index profiles $\boldsymbol{\B^\theta(\X^1)}$ and $\boldsymbol{\B^\theta(\X^2)}$.
Though each agent's Banzhaf index is in range $[0,1]$, the Banzhaf profiles can have different scales, i.e., $\sum_{i\in N}\B^\theta_i(\X^1) \not= \sum_{i\in N}\B^\theta_i(\X^2)$, leading to additional within-vector variance differences.
Therefore, to avoid unfair comparisons, in our experiments, we use the normalized Banzhaf index, i.e., $\frac{\B^\theta_i(\X)}{\sum_{j\in N}\B^\theta_j(\X)}$ for each agent $i\in N$, and we simply refer to it it the Banzhaf index, unless clarification is required by the context.


\begin{algorithm}[t]
\caption{Monte Carlo estimation of Banzhaf pivot probabilities}
\label{alg:exp:banzhaf}
\begin{algorithmic}[1]
\Require Stake vector $x\in\mathbb{R}^n_{>0}$, quota grid $\{\theta_j\}_{j=1}^Q$, sample number $R$
\Ensure Estimated pivotal probabilities $\widehat{P}_{i,j}$
\For{$t=1$ to $R$}
    \State Sample $T^{(t)}_i\sim\mathrm{Bernoulli}(1/2)$ independently
    \State $S^{(t)}\gets\sum_i T^{(t)}_i x_i$
\EndFor
\For{each agent $i$ and quota $\theta_j$}
    \State $\widehat{P}_{i,j}\gets
    \frac{1}{R}\sum_{t=1}^R
    \mathbf{1}\!\left[
    S^{(t)}-T^{(t)}_i x_i < \theta_j
    \wedge
    S^{(t)}-T^{(t)}_i x_i + x_i \ge \theta_j
    \right]$
\EndFor
\State \Return $\widehat{P}$
\end{algorithmic}
\end{algorithm}

\paragraph{Computing the Banzhaf Index.}
For a fixed stake profile and quota, the Banzhaf index is defined as a pivot
probability over coalitions (Definition \ref{def:banzhaf}).
Though the exact Banzhaf index of each agent can be computed in polynomial time $O(n^2s_{\max})$ where $s_{\max}$ is the maximal individual stake number in the stake profile $\s$ \cite{chalkiadakis2022computational}, it remains considerably time consuming for the computation of inter-profile variance, since we would need to compute $100n$ Banzhaf indices for each parameter setting.
To circumvent this issue we employ the Monte Carlo method (Algorithm~\ref{alg:exp:banzhaf}).
The algorithm was introduced and studied in \cite{bachrach2010approximating} and has been already used in the literature for similar experiments (e.g., \cite{zhang2021power}). Algorithm~\ref{alg:exp:banzhaf} implements $rp = 15,000$ iterations to ensure a tight confidence interval on the obtained approximation.
Notice that in Figures \ref{fig:alpha_1_5} and \ref{fig:aggregate-3plots}, due to this approximation process, when $\theta$ approaches $0$ or $1$, the estimated Banzhaf indices are almost $0$ because of the low probability of sampling a pivotal coalition, especially for large $N$.



\paragraph{Empirical data.}
To calibrate the simulation study to real-world conditions, we use also empirical stake
data from Fund~13 of Project Catalyst.\footnote{\url{https://projectcatalyst.io/}}
Among all registered stake addresses, only a subset actively participated in the
vote; in particular, we identify approximately $7{,}050$ addresses that cast a
vote.
Fitting and descriptive analysis are performed on the raw (unnormalized) stake
values of these voting addresses. 

Table~\ref{tab:stake-summary} reports summary statistics of the empirical stake
distribution.
The data exhibit extreme right skew and high concentration, with a small number
of addresses holding a large fraction of total stake.
Moreover, from the official Fund~13 results, the fraction of stake required to
fund the last accepted project allows us to infer the effective quota used in
practice, which is approximately $7\%$ with minor variation across categories.

\begin{table}[t]
\centering
\label{tab:stake-summary}
\begin{tabular}{l|r}
\toprule
\multicolumn{1}{c}{Statistic} & Value \\
\midrule
Number of registered addresses & 61{,}092 \\
Minimum stake & 25 \\
Median stake & 3{,}604 \\
Mean stake & 78{,}628 \\
Maximum stake & 182{,}250{,}000 \\
\bottomrule
\end{tabular}
\vspace{5pt}
\caption{Summary statistics of empirical stake weights in Fund~13 of Project Catalyst}
\label{tab:stake-summary}
\end{table}

\paragraph{Parameter settings.}
To account for stake heterogeneity, we independently sample $M=100$ stake profiles for each parameter setting, and approximate the mean of the power-stake ratio and within-vector variance by Algorithm \ref{alg:exp:profiles}.
For each draw, we evaluate the within-vector variance over a grid of $101$ evenly spaced quota values in $[0,1]$, except for Figures \ref{fig:boxplot-fixed-q-mean} and \ref{fig:boxplot-fixed-q-var} where the quota $\theta=0.07$ reflects the practical quota in Project Catalyst data.
We perform experiments for population sizes
$n\in\{30,40,60,80\}$ and additionally conduct a parameter sweep over stake distribution parameters at fixed $n=50$.
All reported results correspond to finite-sample averages across the $M$ stake
draws.

The random weight profiles are constructed by normalising randomly sampled stake profiles, where each individual stake number is drawn from a Gamma distribution as depicted in Section \ref{sec:distribution_assumption}.
We use three different parameter sets for the underlying Gamma distribution, namely, $(\alpha = 1, \beta = 1)$ and $(\alpha = 5, \beta = 1)$ for comparison against Section \ref{sec:numerical_analysis}, and $(\hat{\alpha} = 0.273568, \beta = 1)$.
The last parameter set $(\hat{\alpha} = 0.273568, \beta = 1)$ is obtained by fitting the Project Catalyst data to a Gamma distribution (Algorithm \ref{alg:gamma-fit}).
Note that we set $\beta = 1$ for all Gamma distributions due to the fact that this scale parameter does not influence the normalized weight profile distribution (see details in Section \ref{sec:distribution_assumption}).
For the Gamma distribution fitting the Project Catalyst data, we provide the simulated single-agent variance results in Appendix \ref{sec:additional_experiment}, which show consistent trends as in Section \ref{sec:numerical_analysis}.





All code used to generate the results is publicly available on Github.\footnote{
Link omitted for double-blind review.}

\begin{algorithm}[t]
\caption{Evaluation over sampled stake profiles}
\label{alg:exp:profiles}
\begin{algorithmic}[1]
\Require Number of agents $n$; Gamma distribution parameters $(\alpha,\beta)$; stake draws $M$;
coalition samples $R$; quota grid $\{\theta_j\}_{j=1}^Q$
\Ensure $\{\widehat{\mathbb{E}}_M[\mu(\theta_j)]\}_{j=1}^Q$,
        $\{\widehat{\mathbb{E}}_M[v(\theta_j)]\}_{j=1}^Q$
\For{$m=1$ to $M$}
    \State Sample $w_i^{(m)}\sim\Gamma(\alpha,\theta)$ and normalize to $x^{(m)}$
    \State Compute the Banzhaf indices using Algorithm~\ref{alg:exp:banzhaf} and normalize to $\B^{\theta(m)}_i$
    \For{$j=1$ to $Q$}
        \State $R_i^{(m)}(\theta_j)\gets \B^{\theta(m)}_i/x_i^{(m)}$
        \State $\mu^{(m)}(\theta_j)\gets \frac{1}{n}\sum_i R_i^{(m)}(\theta_j)$
        \State $v^{(m)}(\theta_j)\gets \frac{1}{n}\sum_i (R_i^{(m)}(\theta_j)-\mu^{(m)}(\theta_j))^2$
    \EndFor
\EndFor
\For{$j=1$ to $Q$}
    \State $\widehat{\mathbb{E}}_M[\mu(\theta_j)]
           \gets \frac{1}{M}\sum_m \mu^{(m)}(\theta_j)$
    \State $\widehat{\mathbb{E}}_M[v(\theta_j)]
           \gets \frac{1}{M}\sum_m v^{(m)}(\theta_j)$
\EndFor
\end{algorithmic}
\end{algorithm}

\begin{algorithm}[t]
\caption{Gamma distribution fitting for empirical stake data (MLE)}
\label{alg:gamma-fit}
\begin{algorithmic}[1]
\Require Empirical stake weights $\{\tilde{w}_i\}_{i=1}^N$, $\tilde{w}_i>0$
\Ensure Fitted parameters $(\hat{\alpha},\hat{\beta})$
\State Define the log-likelihood
\[
\ell(\alpha,\beta)
=
\sum_{i=1}^N
\left[
(\alpha-1)\log \tilde{w}_i
-
\frac{\tilde{w}_i}{\beta}
-
\alpha\log\beta
-
\log\Gamma(\alpha)
\right]
\]
\State Numerically maximize $\ell(\alpha,\beta)$ over $\alpha>0$, $\beta>0$
\State \Return $(\hat{\alpha},\hat{\beta})$
\end{algorithmic}
\end{algorithm}

\subsection{Results of Experiments}
\label{sec:experiments:outputs}

Figure \ref{fig:alpha_1_5} shows the expectations of all agents' power-stake ratios and the within-vector variances when weight profiles are drawn from Gamma distributions with shape parameter $\alpha =1$ and $5$.
Most importantly, contrasting the analytical single-agent variances in Figures \ref{fig:analytical_variance_alpha1} and \ref{fig:analytical_variance_alpha5}, the within-vector variance tends to monotonically decrease from extreme quotas ($\theta = 0$ and $1$) to $\theta = 0.5$, for both parameters $\alpha = 1$ and $5$.
It is important to notice that this observation does not violate the analytical results, since the single-agent variance is a more strict indicator for power balance.
Instead, Figures \ref{fig:var_alpha_1} and \ref{fig:var_alpha_5} complement the analytical results by showing that the average power balance becomes better when the quota moves towards $\theta = 0.5$, showing the potential of the weighted majority rule w.r.t. power balance.
For power-stake ratio means, the values are significantly lower than those in Figures \ref{fig:analytical_mean_alpha1} and \ref{fig:analytical_mean_alpha5} because of the normalisation of the Banzhaf indices.

Figure \ref{fig:aggregate-3plots} shows the simulation trends by fitting the Gamma distribution to the real-world Project Catalyst data.
Figures \ref{fig:intergroup-multi-n-mean} and \ref{fig:intergroup-multi-n-var} present similar trends as Figure \ref{fig:alpha_1_5}.
Especially, the within-vector variance tends to reduce as the quota moves to $\theta=0.5$, however, this trend is milder compared to larger values of $\alpha$.
It is worth noticing that the within-vector variances for $\alpha = 0.273568$ is higher than those for $\alpha = 1$ and $\alpha = 5$.
We conjecture that this happens because of a specific feature of the weight profile distribution.
When $\alpha$ is lower, more agents are assigned with a small weight, and this feature worsens power balance, since a small group of agents can easily accrue voting power even though they do not hold a large weight.

Lastly, combining with Figures \ref{fig:boxplot-fixed-q-mean} and \ref{fig:boxplot-fixed-q-var}, we can observe that the voting mechanism of Project Catalyst is not power-balanced for two reasons: (1) its imbalanced stake distribution induces voting power accrual; (2) its low quota further exacerbates power imbalances.

\begin{figure}[t]
\centering
\begin{subfigure}[t]{0.49\linewidth}
\centering
\includegraphics[width=\linewidth]{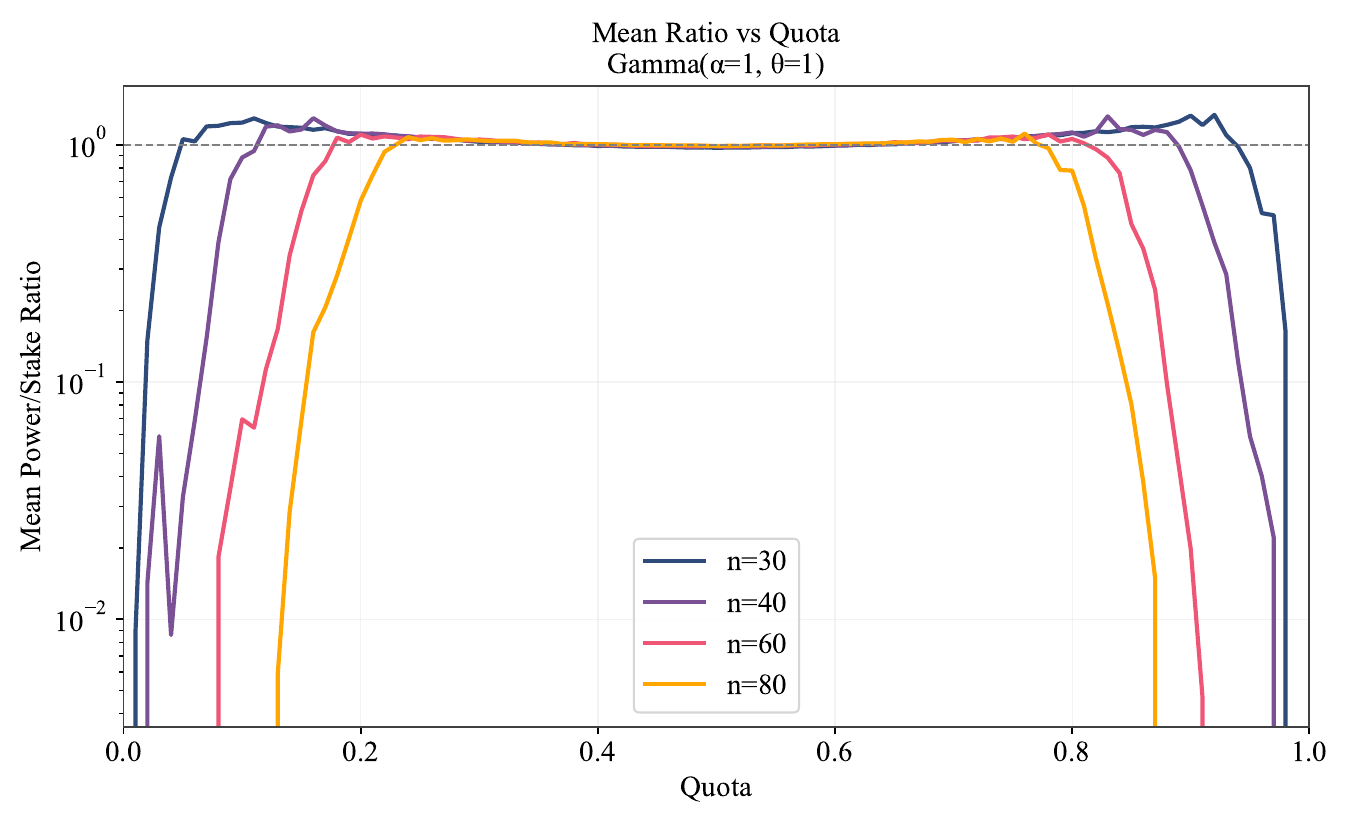}
\caption{Power-stake ratio mean ($\alpha = 1$).}
\label{fig:mean_alpha_1}
\end{subfigure}\hfill
\begin{subfigure}[t]{0.49\linewidth}
\centering
\includegraphics[width=\linewidth]{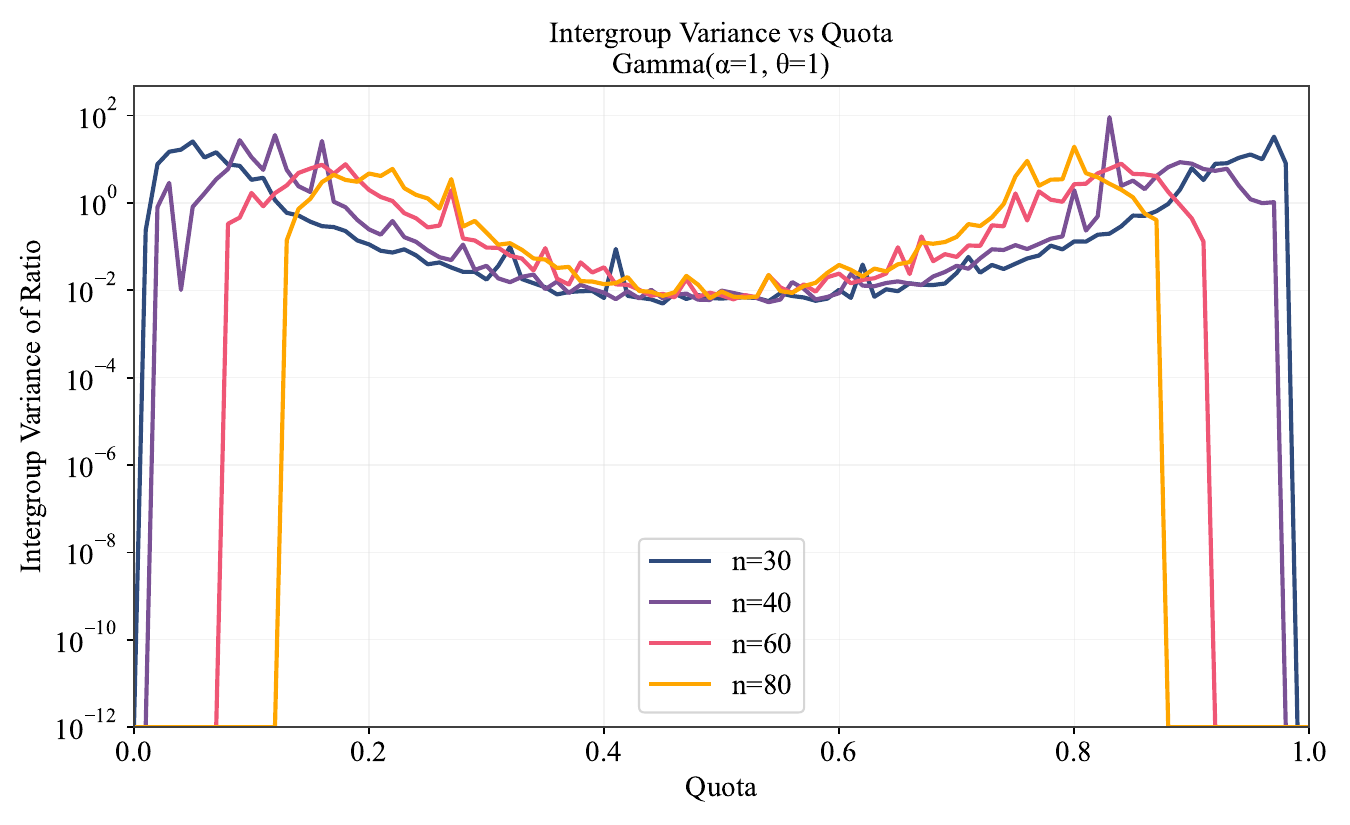}
\caption{Within-vector variance ($\alpha = 1$).}
\label{fig:var_alpha_1}
\end{subfigure}
\begin{subfigure}[t]{0.49\linewidth}
\centering
\includegraphics[width=\linewidth]{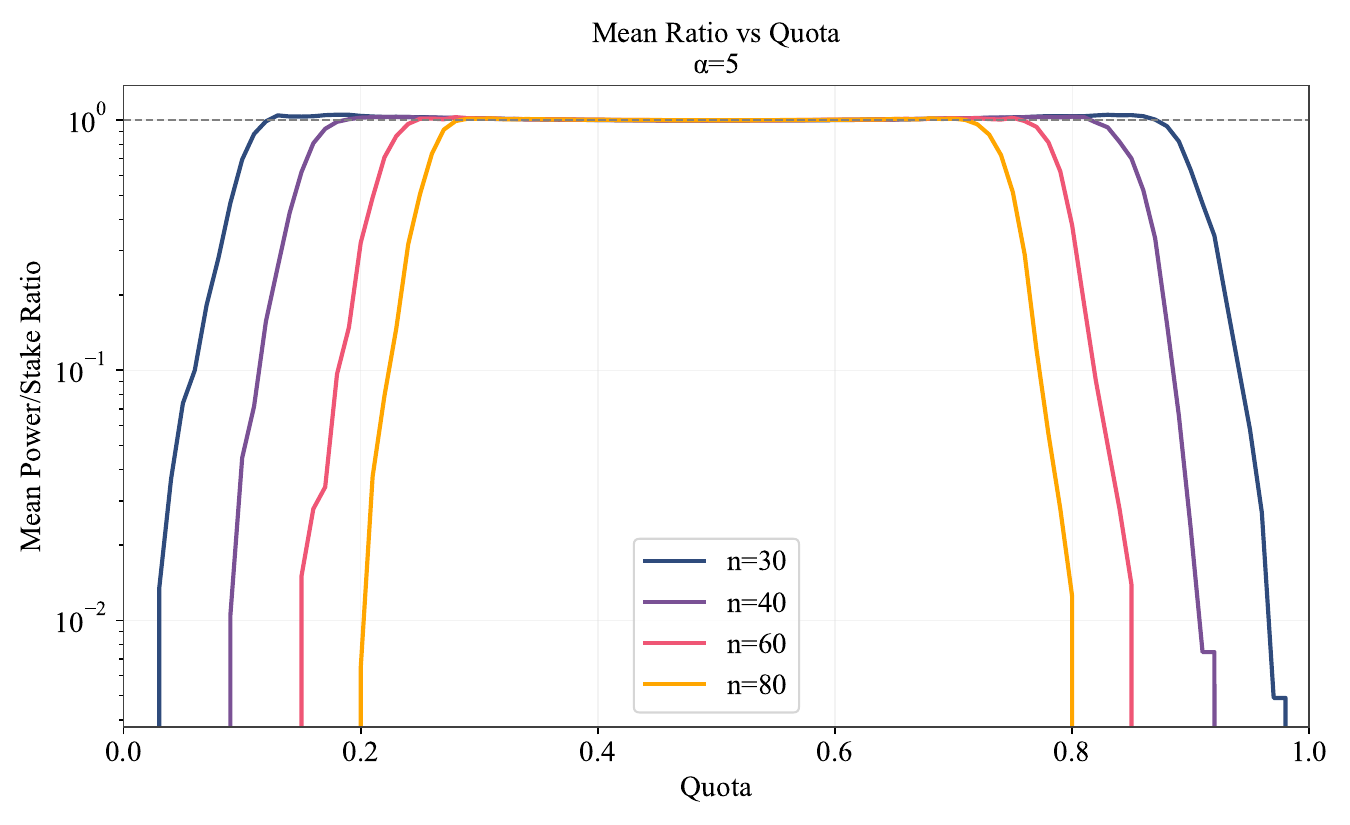}
\caption{Power-stake ratio mean ($\alpha = 5$).}
\label{fig:mean_alpha_5}
\end{subfigure}\hfill
\begin{subfigure}[t]{0.49\linewidth}
\centering
\includegraphics[width=\linewidth]{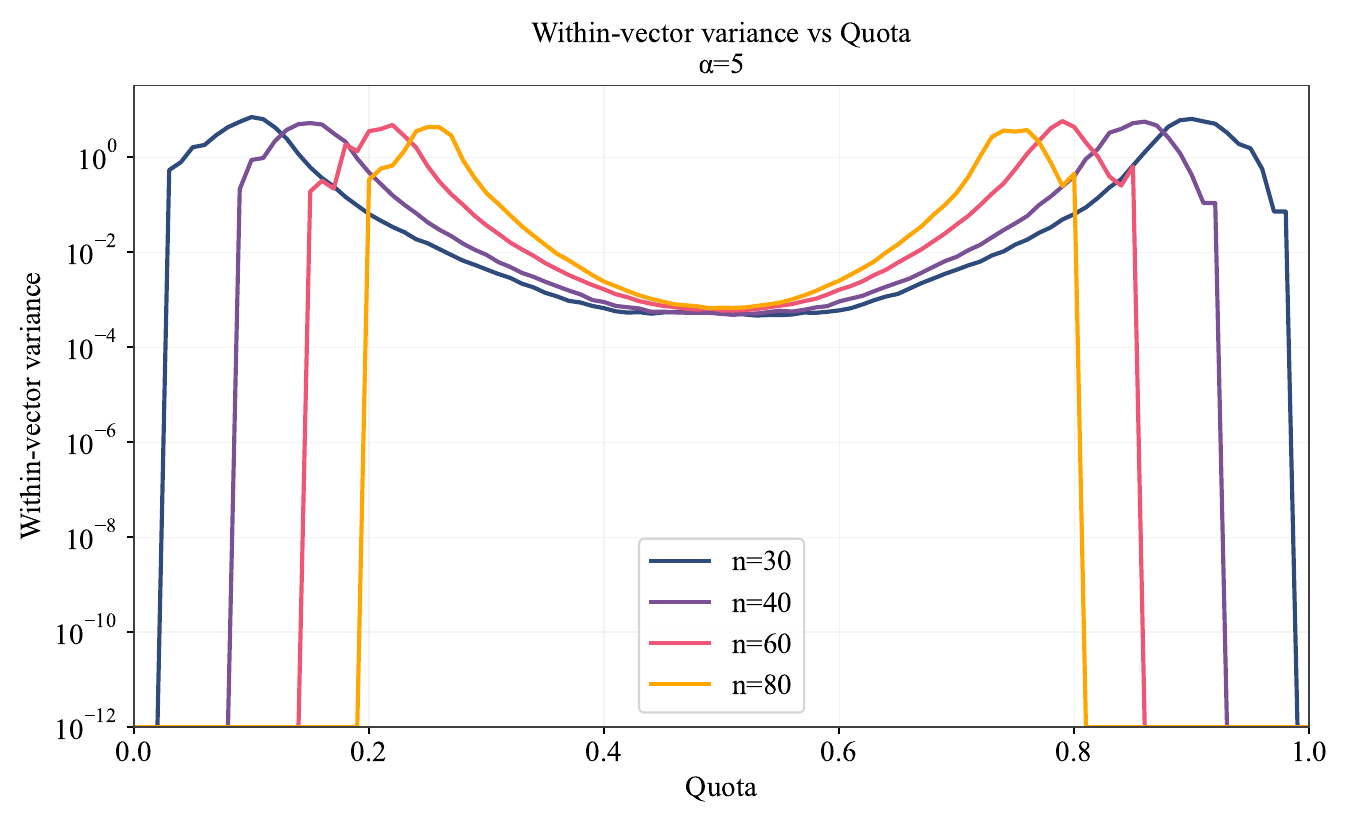}
\caption{Within-vector variance ($\alpha = 5$).}
\label{fig:var_alpha_5}
\end{subfigure}
\caption{Power-stake ratio means and within-vector variances for $\alpha = 1$ and $5$ across quotas for $n\in\{30,40,60,80\}$.}
\label{fig:alpha_1_5}
\end{figure}

  \begin{figure*}[t]
    \centering
    \begin{subfigure}[t]{0.49\textwidth}
      \centering
      \includegraphics[width=\linewidth]{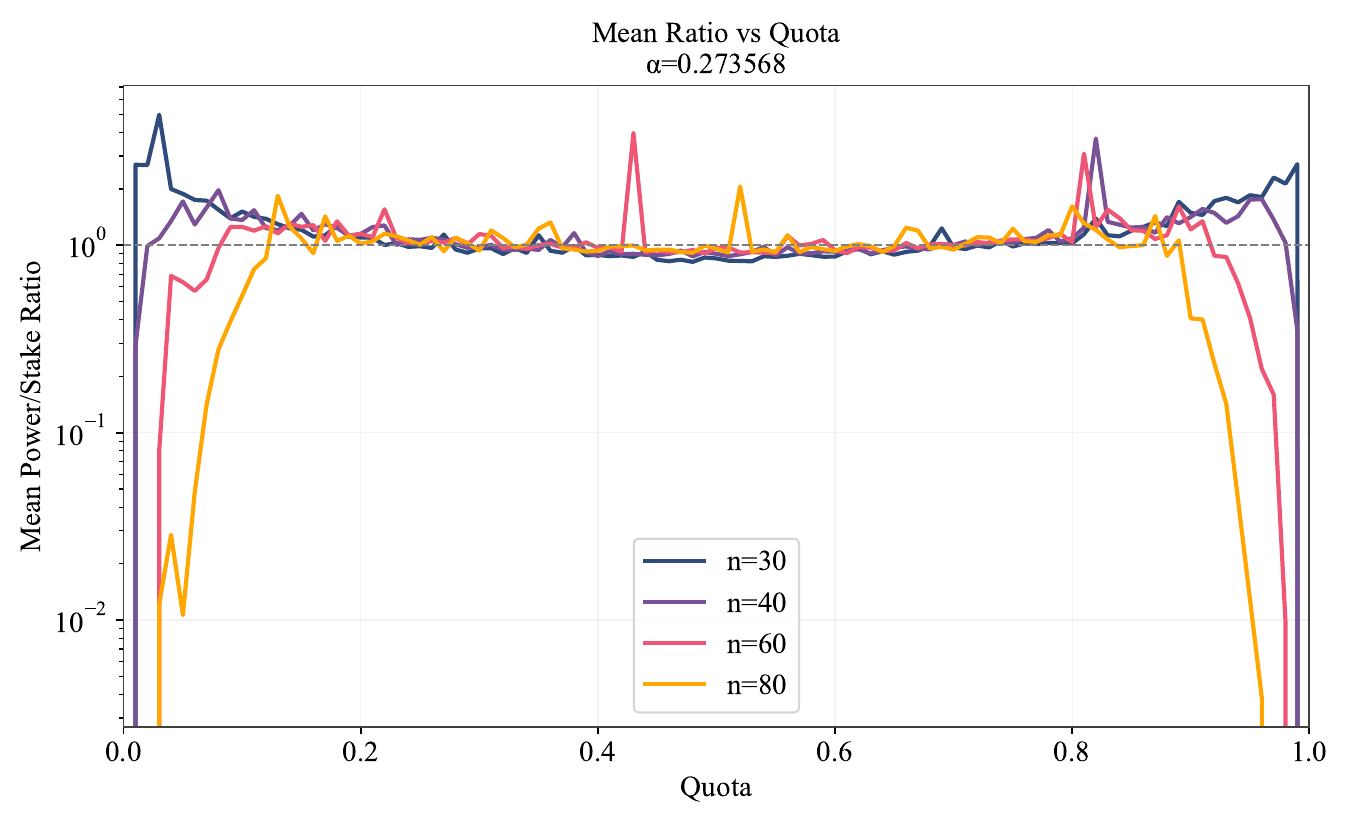}
       \caption{Inter-profile mean ratio vs quota ($n\in\{30,40,60,80\}$).}
      \label{fig:intergroup-multi-n-mean}
    \end{subfigure}\hfill
    \begin{subfigure}[t]{0.49\textwidth}
      \centering
      \includegraphics[width=\linewidth]{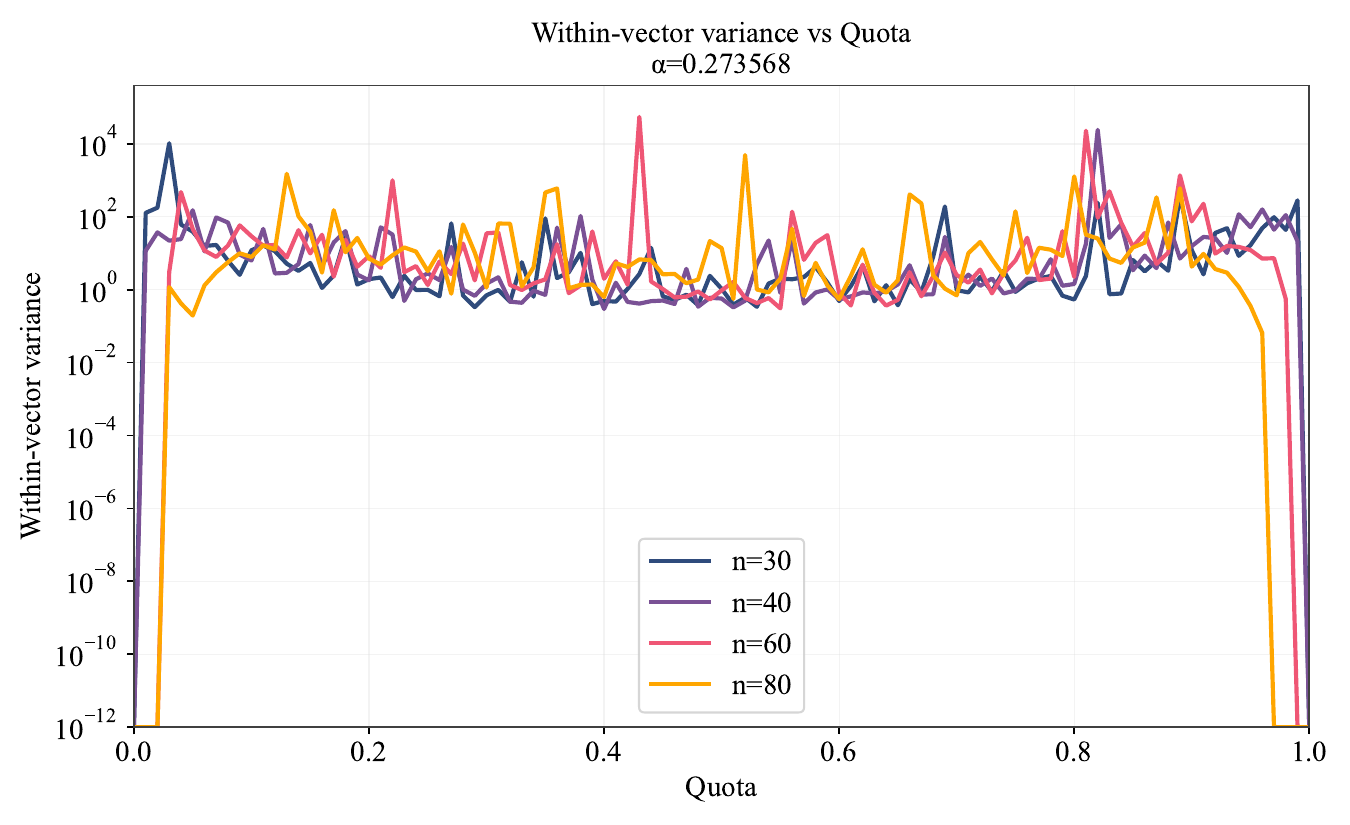}
      \caption{Inter-profile variance vs quota ($n\in\{30,40,60,80\}$).}
      \label{fig:intergroup-multi-n-var}
    \end{subfigure}\\[0.6em]
        \begin{subfigure}[t]{0.49\textwidth}
      \centering
      \includegraphics[width=\linewidth]{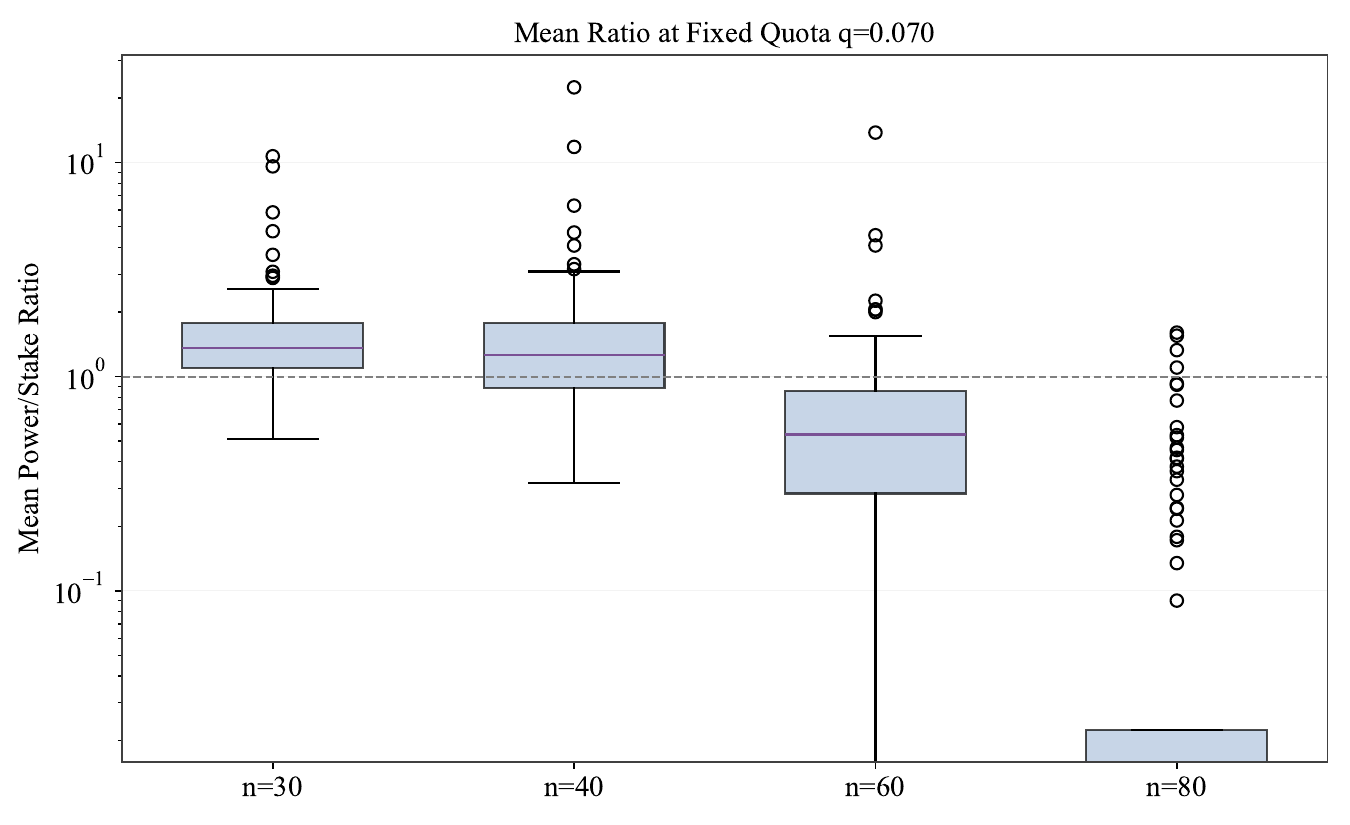}
      \caption{Mean distribution at fixed quota $\theta=0.07$.}
      \label{fig:boxplot-fixed-q-mean}
    \end{subfigure}\hfill
    \begin{subfigure}[t]{0.49\textwidth}
      \centering
      \includegraphics[width=\linewidth]{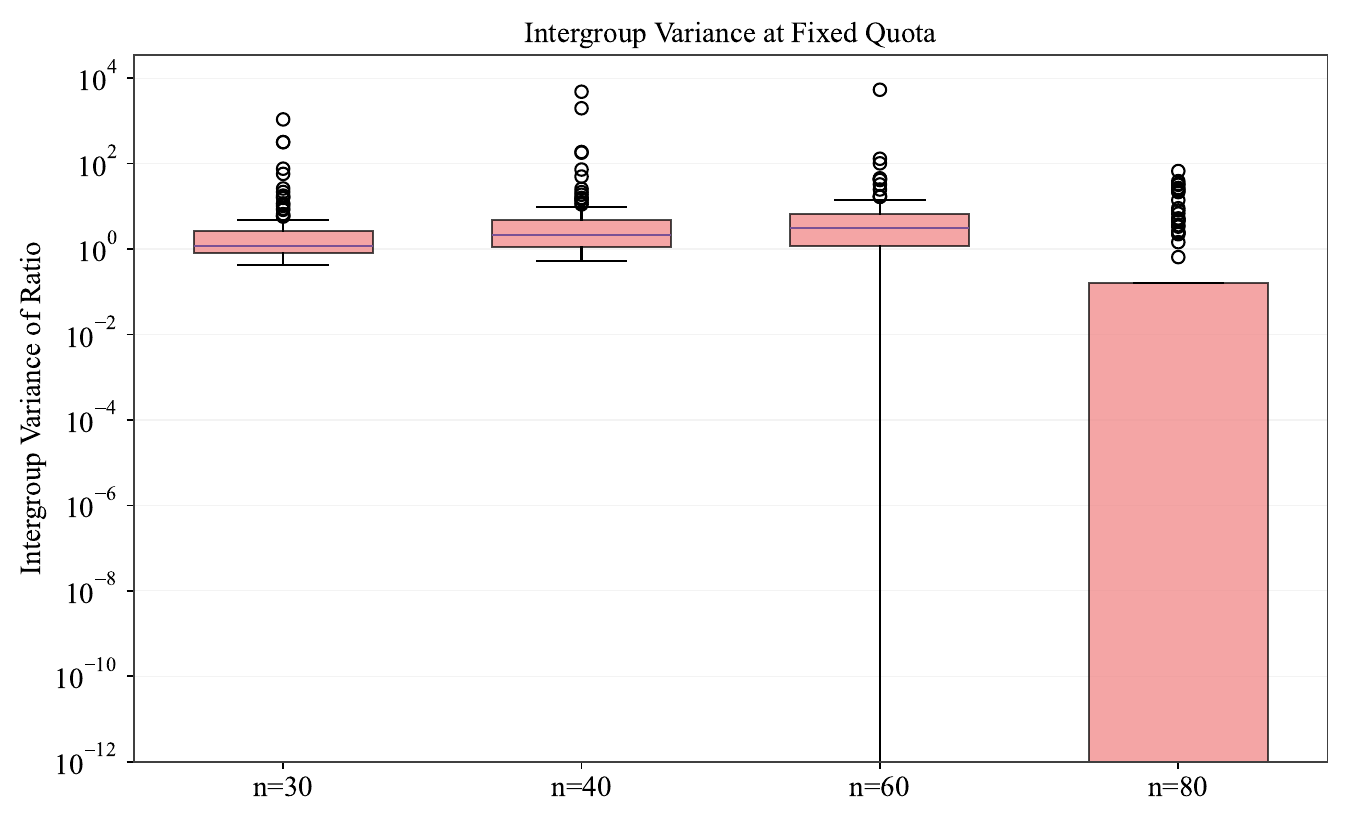}
      \caption{Variance distribution at fixed quota $\theta=0.07$.}
      \label{fig:boxplot-fixed-q-var}
    \end{subfigure}\\[0.6em]

    \caption{Aggregate behavior of normalized power--stake ratios under varying quotas and stake distributions.}
    \label{fig:aggregate-3plots}
  \end{figure*}



\section{Conclusions}

That users concentrating extremely large stakes---so-called `whales'---may constitute a problem in blockchain governance is a widespread realization. This is of special relevance in complex decision-making settings such as budgeting. Our results provide detail and further underpinning to such an observation, by viewing it as a symptom of a more general phenomenon: the failure of current blockchain governance paradigms (typically based on stake-weighted voting) to correctly translate stake ownership into decision-making influence. We determine the extent of such a failure by: precisely quantifying the variance on power that a single user should expect under broad conditions (Theorem \ref{thm:single_variance}); and by estimating via experiments (Section \ref{sec:experiments} the variance on power distribution that is to be expected across users under realistic conditions.

\medskip

While our work has so far been only diagnostic, in that it only scoped the extent of power imbalances in blockchain governance, future work should focus on approaches that can tackle such imbalances and provide alternative blockchain governance paradigms. A natural direction to be pursued is the adaptation of proportional participatory budgeting methods \cite{los2022proportional,rey2023computational} developed in the context of democratic governance---and hence under the one-person one-vote paradigm---to a sake-weighted setting with sparse ballots. A natural candidate for such an adaptation is, for example, the method of equal shares \cite{peters2021proportional}.

\begin{acks}
We would like to thank Giovanni Gargiulo for his assistance with the Project Catalyst data. We also thank the participants of the Cardano Developer Office Hours and the 1st CRWA Workshop (affiliated with Asiacrypt'25) for their valuable feedback and discussions. This research was supported in part by Cardano's Project Catalyst through Project ID \hlhref{https://projectcatalyst.io/funds/13/cardano-use-cases-concept/proportionality-in-stake-based-voting}{1300163}.
\end{acks}

\newpage

\bibliographystyle{ACM-Reference-Format}
\bibliography{bib}

\newpage
\section{Appendices}
\appendix

\section{Additional Analytical Results}
\begin{figure}[t]
\centering
\begin{subfigure}[t]{0.49\linewidth}
\centering
\includegraphics[width=\linewidth]{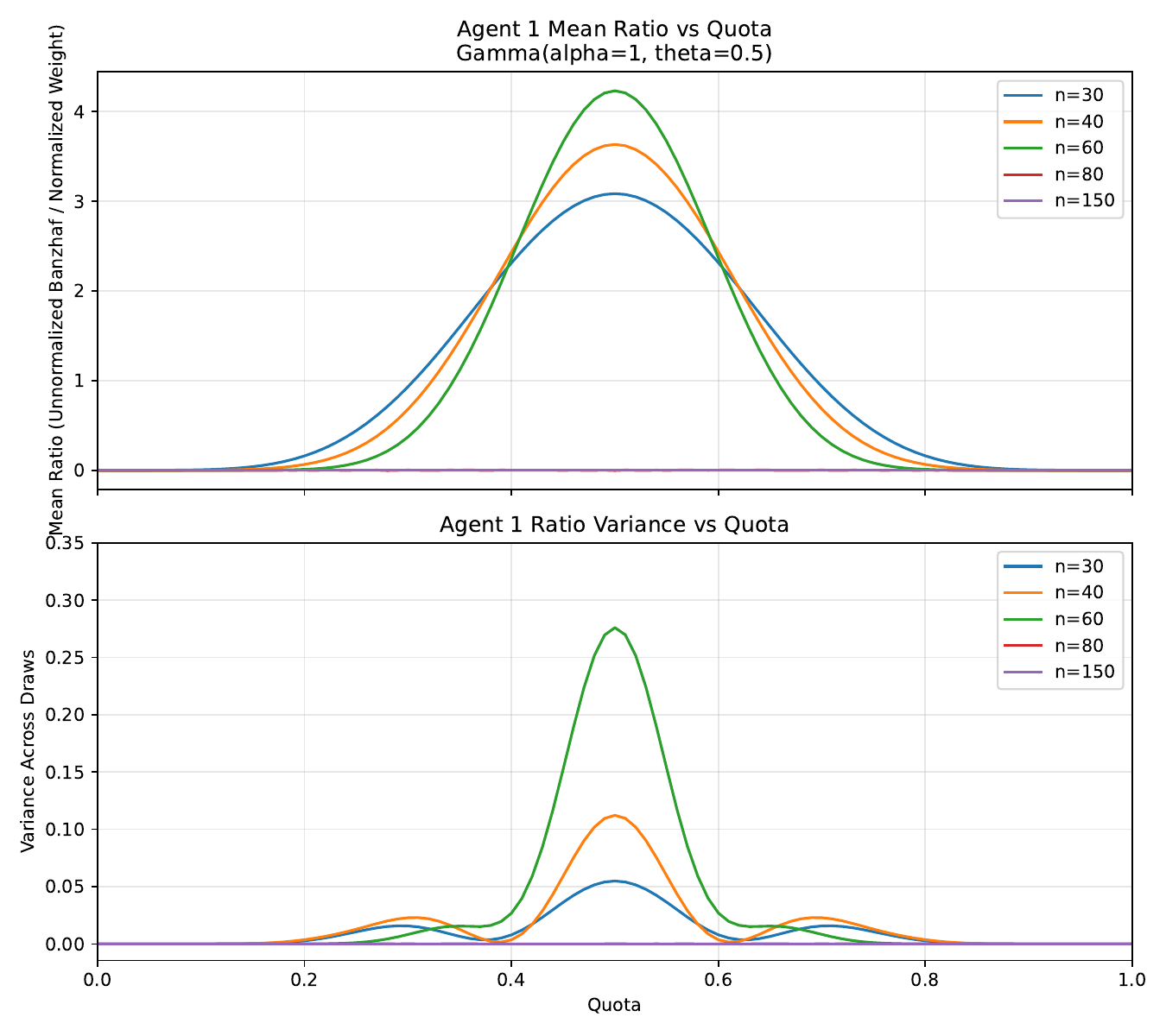}
\caption{$\alpha = 1$.}
\label{fig:verify_1}
\end{subfigure}\hfill
\begin{subfigure}[t]{0.49\linewidth}
\centering
\includegraphics[width=\linewidth]{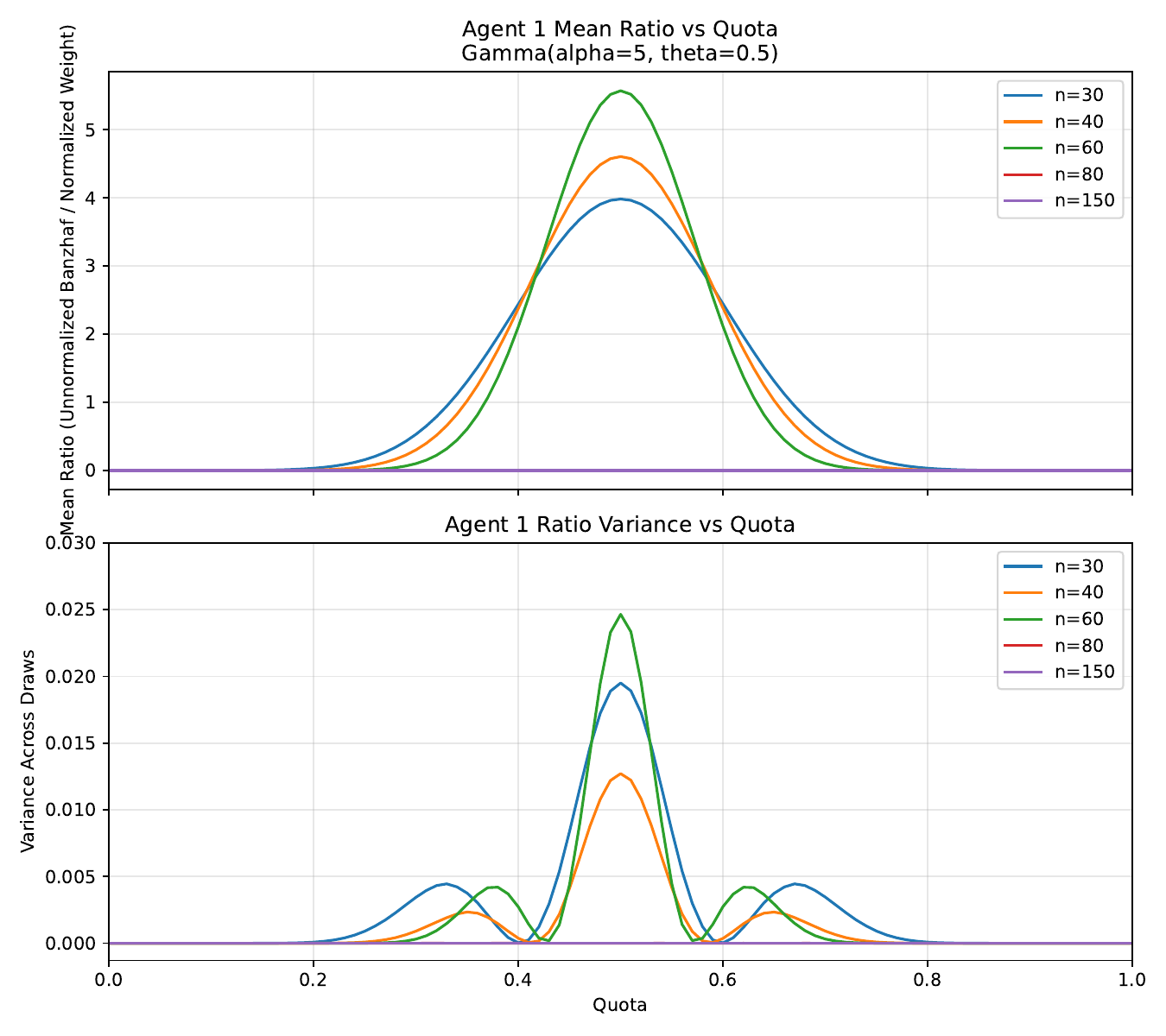}
\caption{$\alpha=5$.}
\label{fig:verify_2}
\end{subfigure}
\caption{Mean and variance of the first agent’s normalized power--stake ratio across quotas for $n\in\{30,40,60,80\}$.}
\label{fig:simulation_verify}
\end{figure}
\subsection{Experimental Verification of Theorem \ref{thm:single_variance}}\label{sec:verify}
We draw the counterpart plots of Figure \ref{fig:analytical_plots} by simulation to verify the correctness of the expressions of single-agent variances.
We use the same parameters: $\alpha = 1$ and $5$, $n=(30, 40, 60, 80, 150)$, and quotas in $(0,1)$ with step of $0.01$.
For each parameter $n$ and $\theta$, the simulation process is: (1) 20 weight profiles are drawn from the distribution; (2) compute the power-stake ratios for agent $1$ w.r.t. the 20 weight profiles; (3) approximate the single-agent variance by computing the variance of the 20 ratios.

Figure \ref{fig:simulation_verify} shows similar trends as those in Figure \ref{fig:analytical_plots}.

\section{Additional Experiment Results}\label{sec:additional_experiment}
Figure \ref{fig:fit_single_agent} shows the expected power-stake ratios and single-agent variances for weight profiles drawn from the Gamma distribution that fits the Project Catalyst data.

\begin{figure}[t]
\centering
\begin{subfigure}[t]{0.49\linewidth}
\centering
\includegraphics[width=\linewidth]{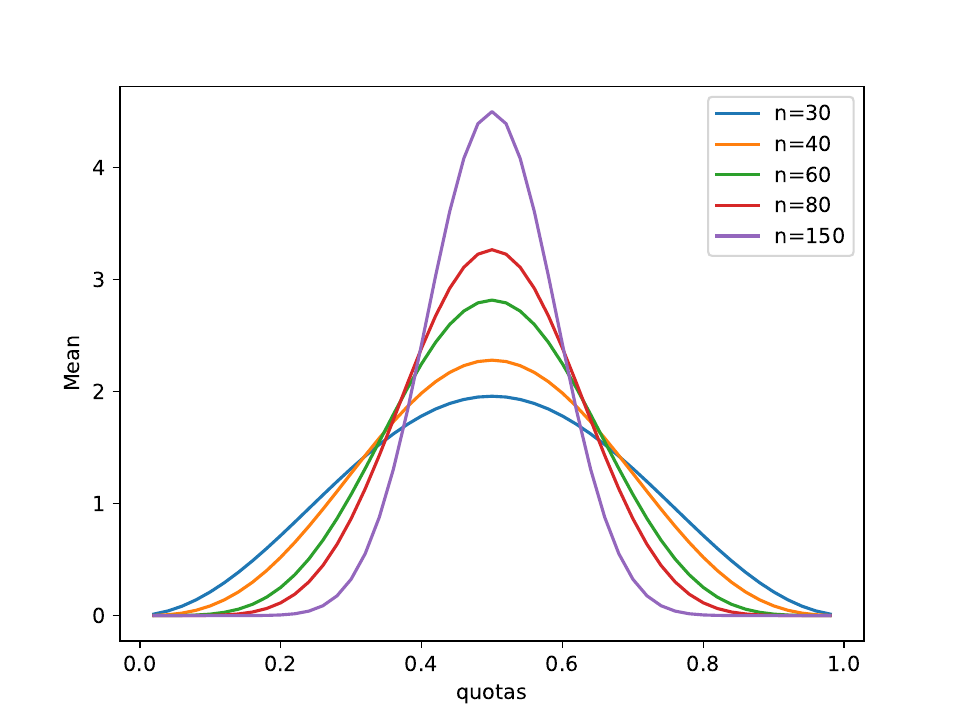}
\caption{Expected power-stake ratio ($\alpha = 0.273568$).}
\label{fig:fit_mean}
\end{subfigure}\hfill
\begin{subfigure}[t]{0.49\linewidth}
\centering
\includegraphics[width=\linewidth]{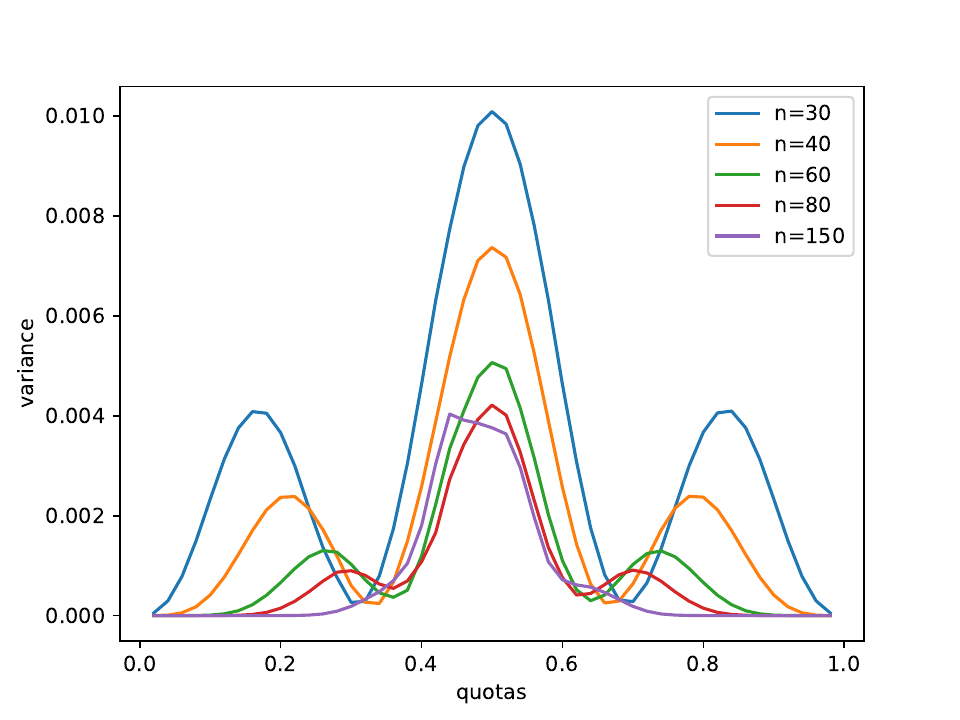}
\caption{Single-agent variance ($\alpha = 0.273568$).}
\label{fig:fit_variance}
\end{subfigure}
\caption{Mean and variance of the first agent’s normalized power--stake ratio across quotas for $n\in\{30,40,60,80\}$.}
\label{fig:fit_single_agent}
\end{figure}

\end{document}